\documentclass{article}

\pdfoutput=1

\usepackage[utf8]{inputenc} % allow utf-8 input
\usepackage[T1]{fontenc}    % use 8-bit T1 fontsw

\usepackage[backref=page]{hyperref}
\renewcommand*{\backref}[1]{}
\renewcommand*{\backrefalt}[4]{%
    \ifcase #1 (Not cited)%
    \or        (Cited on page~#2)%
    \else      (Cited on pages~#2)%
    \fi}

\usepackage{url}            % simple URL typesetting
\usepackage{booktabs}       % professional-quality tables
\usepackage{amsfonts}       % blackboard math symbols
\usepackage{nicefrac}       % compact symbols for 1/2, etc.
\usepackage{microtype}      % microtypography
\usepackage{amsmath}
\usepackage{pifont}
\usepackage{microtype}
\usepackage{graphicx}
\usepackage{caption}
\usepackage{subcaption}
\usepackage{algorithm}
\usepackage{algorithmic}
\usepackage{wrapfig}
\usepackage{comment}
\date{}
\usepackage{cleveref}
\usepackage{multirow}
\usepackage{balance}

\usepackage{algorithm}
\usepackage{algorithmic}

\usepackage{float}
\usepackage{color}
\usepackage{amsmath}

\usepackage{amsthm,amssymb}
\usepackage{tikz}

\usepackage[accepted]{icml2021}

\newcommand{\cmark}{\text{\ding{51}}}
\newcommand{\xmark}{\text{\ding{55}}}

\newtheorem{theorem}{Theorem}[section]

\newtheorem{lemma}[theorem]{Lemma}

\newtheorem{proposition}{Proposition}

\usepackage{varwidth}

\renewcommand{\algorithmicrequire}{\textbf{Input:}}
\renewcommand{\algorithmicensure}{\textbf{Output:}}

\DeclareMathOperator*{\argmax}{arg\,max}

\def\TM{\mathcal{T}}

\def\dNN{\text{NN}}

\def\RR{\mathbb{R}}

\def\NN{\mathbb{N}}

\def\SS{\mathbb{S}}

\def\Pp{\mathcal{P}}
\def\Ss{\mathcal{S}}

\def\Rr{\mathcal{R}}

\def\Tt{\mathcal{T}}

\def\bx{\mathbf{x}}
\def\by{\mathbf{y}}
\def\bz{\mathbf{z}}

\def\dOperation{W_{d_{\Tt_{o}}}}
\def\dIn{W_{d_{\Tt_{\!-}}}}
\def\dOut{W_{d_{\Tt_{\!+}}}}

\global\long\def\bx{\mathbf{x}}

\global\long\def\bz{\mathbf{z}}

\global\long\def\bX{\mathbf{X}}

\global\long\def\by{\mathbf{y}}

\global\long\def\bK{\boldsymbol{K}}

\global\long\def\and{\cap}

\global\long\def\var{\text{Var}}

\global\long\def\cov{\text{Cov}}

\global\long\def\idenmat{\mathbf{I}}

\global\long\def\argmax#1{\underset{_{#1}}{\text{argmax}} }

\icmltitlerunning{Optimal Transport Kernels for Sequential and Parallel Neural Architecture Search}

\begin{document}

\twocolumn[

\icmltitle{Optimal Transport Kernels for Sequential and Parallel \\ Neural Architecture Search}

\begin{icmlauthorlist}

\icmlauthor{Vu Nguyen$\text{}^*$} {amz}
\icmlauthor{Tam Le$\text{}^*$} {riken}
\icmlauthor{Makoto Yamada} {riken,kyoto}
\icmlauthor{Michael A. Osborne } {ox}

\end{icmlauthorlist}

\icmlaffiliation{amz}{Amazon Adelaide (work done prior to joining Amazon)}
\icmlaffiliation{riken}{RIKEN AIP}
\icmlaffiliation{kyoto}{Kyoto University}
\icmlaffiliation{ox}{University of Oxford}

\icmlcorrespondingauthor{Vu Nguyen}{vu@ieee.org}

\icmlkeywords{Machine Learning, ICML}

\vskip 0.3in
]

\printAffiliationsAndNotice{\icmlEqualContribution} % otherwise use the standard text.

% \maketitle

% This is as opposed to the existing optimal transport distance \cite{kandasamy2018neural}.

% The performance of deep neural network depends on its architecture. 

% A complex and non-continuous structure of the networks raises challenges for modeling that the conventional Euclidean space may fail to capture the similarity across the architectures.

% \begin{abstract}
% Neural architecture search (NAS) automates the engineering process for identifying the best design for deep neural network. One of the challenges in searching the complex and non-continuous architectures is to measure their similarity that the conventional Euclidean space may fail to capture. Optimal transport (OT) is resilient to such complex structure by taking the similarity as the minimal cost for transporting a network into another. However, the OT is generally not negative definite, an important property for constructing positive semi definite kernels required in common machine learning frameworks. Therefore, we develop a negative definite variant of OT, tree-Wasserstein (TW), for measuring similarity across architectures. We demonstrate the proposed TW on the Gaussian process surrogate model for sequential and parallel NAS setting. As a secondary contribution, we present a novel parallel NAS approach using quality k-determinantal point process on the GP posterior to select diverse and high-performing architectures from a discrete set of candidates. We empirically demonstrate that our approaches with TW geometry outperform other baselines in both sequential and parallel NAS.
% \end{abstract}

\vspace{-10pt}

\begin{abstract}
Neural architecture search (NAS) automates the design of deep neural networks. One of the main challenges in searching complex and non-continuous architectures is to compare the similarity of networks that the conventional Euclidean metric may fail to capture. Optimal transport (OT) is resilient to such complex structure by considering the minimal cost for transporting a network into another. However, the OT is generally not negative definite which may limit its ability to build the positive-definite kernels required in many kernel-dependent frameworks. Building upon tree-Wasserstein (TW), which is a negative definite variant of OT, we develop a novel discrepancy for neural architectures, and demonstrate it within a Gaussian process (GP) surrogate model for the sequential NAS settings. Furthermore, we derive a novel parallel NAS, using quality k-determinantal point process on the GP posterior, to select diverse and high-performing architectures from a discrete set of candidates. We empirically demonstrate that our TW-based approaches outperform other baselines in both sequential and parallel NAS.
\end{abstract}

%  for transporting not only layer operations, but also network structure across neural architectures.  

% While preserving the elegant properties of OT, TW leads to a positive semi-definite kernel, a required property for modeling with Gaussian process (GP).

% in which we draw a connection to use .
%One of the challenges in optimizing neural architecture is how to characterize similarity.

% Our distance finds the minimal cost of transporting the one network to another. 
%NASBOW := NAS + BO + TW

%valid kernel in Wasserstein geometry, capturing both global (e.g., network structure) and local (e.g., operations in network) information in neural networks

%{Revisiting Bayes Opt NAS with Tree-Wasserstein and Batch Extension}
%{Batch NAS with Tree-Wasserstein and kDPP}

%Permutation invariant.

%\Tam{@Vu: you may need to clarify between similarity and distance (dissimilarity).}

%%%%%%%%%%%%%%%%%%%%%%%%%%%%%%%%%%%%%%%%%
%%%%%%%%%%%%%%%%%%%%%%%%%%%%%%%%%%%%%%%%%
\section{Introduction}
\label{sec:introduction}

% Where we need to provide a NAS system with a dataset and a task (classification, regression, etc), and it will give us the architecture.
Neural architecture search (NAS) is the process of automating architecture engineering to find the best design of our neural network model. This output architecture will perform well for a given dataset. With the increasing interest in deep learning in recent years, NAS has attracted significant research attention \cite{dong2019searching,elsken2019efficient,liu2018hierarchical,liu2019darts,luo2018neural,real2019regularized,real2017large,shah2018amoebanet,suganuma2017genetic,xie2017genetic,yao2020efficient}. We refer the interested readers to the survey \cite{elsken2019neural} for a detailed review of NAS and to the comprehensive list\footnote{\url{https://www.automl.org/automl/literature-on-neural-architecture-search}} for all of the related papers in NAS. %We also summarize some of the key related work in the Appendix \S \ref{sec:related_work}.
%Particularly, NAS aims to find the best architecture using the fewest number of evaluations by following a search strategy that will maximize the performance over the dataset.

%[to be rephrase] In the world of artificial intelligence, everyone is immersed in the field of deep learning. Day by day, deep learning takes place in the position of every machine learning models. The capacity of learning like a human brain made it the best fit to be used in the field of artificial intelligence. The last few years have seen much success of deep neural networks in many challenging applications such as speech recognition, computer vision, image recognition, machine translation etc., Along this success there is also a problem of designing the best architecture and features for the network. The designing procedure still requires a lot of experts of knowledge.

%There has been a surge of interest in methods for NAS [...] which fall into four categories:  evolutionary algorithms (EA), reinforcement learning (RL), gradient-based approach and Bayesian optimization (BO). We discuss them in detail in the Appendix due to space constraints. None of the above methods have been designed with a focus on the expense of evaluating a neural network, with an emphasis on being judicious in selecting which architecture to try next. 
Bayesian optimization (BO) utilizes a probabilistic model, particularly Gaussian process (GP) \cite{Rasmussen_2006gaussian}, for determining future evaluations and its evaluation efficiency makes it well suited for the expensive evaluations of NAS. However, the conventional BO approaches \cite{Shahriari_2016Taking,Snoek_2012Practical} are not suitable to capture the complex and non-continuous designs of neural architectures. Recent work \cite{kandasamy2018neural} has considered optimal transport (OT) for measuring neural architectures. This views two networks as logistical \textit{suppliers} and \textit{receivers}, then optimizes to find the minimal transportation cost as the distance, i.e., similar architectures will need less cost for transporting and vice versa. However, the existing OT distance for architectures, such as OTMANN \cite{kandasamy2018neural}, do not easily lend themselves to the creation of the positive semi-definite (p.s.d.) kernel (covariance function) due to the  indefinite property of OT \cite{PeyreCuturiBook} (\S8.3). It is critical as the GP is not a valid random process when the covariance function (kernel) is not p.s.d. (see  Lem. \ref{lem:GP_psd}). In addition, there is still an open research direction for \textit{parallel NAS} where the goal is to select multiple high-performing and diverse candidates from a \textit{discrete} set of candidates for parallel evaluations. This discrete property makes the parallel NAS interesting and different from the existing batch BO approaches \cite{Desautels_2014Parallelizing, Gonzalez_2015Batch}, which are typically designed to handle continuous observations.

We propose a negative definite tree-Wasserstein (TW) distance for neural network architectures based on a novel design which captures both global and local information via $n$-gram and indegree/outdegree representations for networks. In addition, we propose the k-determinantal point process (k-DPP) quality for selecting diverse and high-performing architectures from a \textit{discrete} set. This discrete property of NAS makes k-DPP ideal in sampling the choices overcoming the greedy selection used in the existing batch Bayesian optimization \cite{Desautels_2014Parallelizing,Gonzalez_2015Batch,wang2018batched}. At a high level, our contributions are three-fold as follows:
\begin{itemize}
    \item A TW distance with a novel design for capturing both local and global information from architectures which results in a p.s.d. kernel while the existing OT distance does not.
    \item A demonstration of TW as the novel GP covariance function for sequential NAS. 
    \item A parallel NAS approach using k-DPP for selecting diverse and high-quality architectures from a discrete set.
    
\end{itemize}

% Particularly, we formulate the kernel matrix for k-DPP by connecting k-DPP to a GP. This connection is useful in learning the suitable kernel matrix for k-DPP without the need of tuning hyperparameters for this kernel.

% We connect the k-DPP quality to the GP surrogate model in which the GP predictive uncertainty encodes k-DPP conditioning for exploration and the GP predictive mean encodes for the quality for k-DPP. This connection is useful in learning the kernel matrix for k-DPP.

% List all advantages:

% \textbf{Comparing to Bananas:}
% In term of GP vs NN: GP closed-form update without iterative approximation in neural network (via back-propagation).
% Closed-form uncertainty estimation. Generalising using a few observations.

%In term of architectures vs path-based encoding: can scale well to more number of nodes, layers.

% \textbf{Comparing to NASBOT.}: p.s.d. covariance matrix, less hyperparameters in defining the cost.

%\textbf{Comparing to graph kernel}: we have layer masses, operation type. Edit distance is not suitable for different sizes of architectures.

% OT:  To measure the difference between two neural network architectures, we estimate the minimal cost needed to transform from one to another.

%For a discrete probability measure $\mu$, denote $|\mu|$ for the number of supports of $\mu$. 

\section{Tree-Wasserstein for Neural Network Architectures}\label{sec:problem}

We first argue that the covariance matrices associated with a kernel function of Gaussian process (GP) and k-DPP need to be positive semi-definite (p.s.d.) for a valid random process in Lemma \ref{lem:GP_psd}. We then develop tree-Wasserstein (TW) \cite{do2011sublinear, le2019tree}, the negative definite variant of optimal transport (OT), for measuring the similarity of architectures. Consequently, we can build a p.s.d. kernel upon OT geometry for modeling with GPs and k-determinantal point processes (k-DPPs).

\begin{lemma}\label{lem:GP_psd}
If a covariance function $k$ of a Gaussian process is not positive semi-definite, the resulting GP is not a valid random process.
\end{lemma}

%If $k$ is not positive-semidefinite, we can construct a case where the value of the process at a particular point has negative variance. 
Proof of Lemma \ref{lem:GP_psd} is placed in the Appendix \S\ref{sec:proof_lem:GP_psd}. 

\subsection{Tree-Wasserstein} \label{sec:reviewTW}

We give a brief review about OT, tree metric, tree-Wasserstein (TW) which are the main components for our NAS framework. We denote $[n] = \{1, 2, \ldots, n \}$, $\forall n \in \mathbb{N}_{+}$. Let $(\Omega, d)$ be a measurable metric space. For any $x \in \Omega$, we use $\delta_x$ for the Dirac unit mass on $x$.

\paragraph{Optimal transport.} OT, a.k.a. Wasserstein, Monge-Kantorovich, or Earth Mover's distance, is the flexible tool to compare probability measures \cite{PeyreCuturiBook, villani2003topics}. Let $\omega$, $\nu$ be Borel probability distributions on $\Omega$ and $\Rr(\omega, \nu)$ be the set of probability distributions $\pi$ on $\Omega \times \Omega$ such that $\pi(B \times \Omega) = \omega(B)$ and $\pi(\Omega \times B') = \nu(B')$ for all Borel sets $B$, $B'$. The 1-Wasserstein distance $W_d$ \cite{villani2003topics} (p.2) between $\omega$ and $\nu$ is defined as:
\begin{equation}\label{equ:OTprob}
W_d(\omega, \nu) = \inf_{\pi \in \Rr(\omega, \nu)}\int_{\Omega \times \Omega} d(x, z) \pi(\textnormal{d}x, \textnormal{d}z),
\end{equation}
where $d$ is a ground metric (i.e., cost metric) of OT.

\paragraph{Tree metrics and tree-Wasserstein.} A metric $d:\Omega\times\Omega\rightarrow \RR_{+}$ is a \textit{tree metric} if there exists a tree $\Tt$ with positive edge lengths such that $\forall x \in \Omega$, then $x$ is a node of $\Tt$; and $\forall x, z \in \Omega$, $d(x, z)$ is equal to the length of the (unique) path between $x$ and $z$~\cite{semple2003phylogenetics} (\S7, p.145--182). 

Let $d_{\TM}$ be the tree metric on tree $\Tt$ rooted at $r$. For $x, z \in \Tt$, we denote $\Pp(x, z)$ as the (unique) path between $x$ and $z$. We write $\Gamma(x)$ for a set of nodes in the subtree of $\Tt$ rooted at $x$, defined as $\Gamma(x) = \bigl\{z \in \Tt \mid x \in \Pp(r, z) \bigr\}$. For edge $e$ in $\Tt$, let $v_e$ be the deeper level node of edge $e$ (the farther node to root $r$), and $w_e$ be the positive length of that edge.  %, as illustrated in Fig.~\ref{fg:TW}

Tree-Wasserstein (TW) is a special case of OT whose ground metric is a tree metric \cite{do2011sublinear, le2019tree}. Given two measures $\omega$, $\nu$ supported on tree $\Tt$, and setting the tree metric $d_{\TM}$ as the ground metric, then the TW distance $W_{d_{\TM}}$ between $\omega$ and $\nu$ admits a closed-form solution as follows:
%\vcom{Michael Cohen, do you know how to fix the /\ left and /\ right to make the brackets bigger in the equation? it doesnot work as expected here.}
\begin{align}
W_{d_{\TM}}(\omega, \nu) = \sum_{e \in \Tt} w_e \bigl| \omega \bigl(\Gamma(v_e) \bigr) - \nu \bigl( \Gamma(v_e) \bigr) \bigr|\label{equ:OT_LT},
\end{align}
where $\omega(\Gamma(v_e))$ is the total mass of the probability measure $\omega$ in the subtree $\Gamma(v_e)$ rooted at $v_e$. It is important to note that we can derive p.s.d. kernels on tree-Wasserstein distance $W_{d_{\TM}}$~\cite{le2019tree}, as opposed to the standard OT $W_d$ for general ground metric $d$~\cite{PeyreCuturiBook}.% This motivates us to develop a new (pseudo-)distance for measuring the similarity between neural networks based on TW. %On the other hand, the standard optimal transport $W_d$ is not negative definite for general $d$~\cite{PeyreCuturiBook} (\S8.3) that makes it restricted for modeling with Gaussian process or other machine learning techniques requiring a p.s.d. property. 

%\subsection{Tree-(Sliced)-Wasserstein for Neural Networks}\label{sec:TW4NN}
\subsection{Tree-Wasserstein for Neural Networks}\label{sec:TW4NN}
We present a new approach leveraging the tree-Wasserstein for measuring the similarity of neural network architectures. We consider a neural network architecture $\bx$ by $\left( \Ss^{o}, A \right)$ where $\Ss^{o}$ is a multi-set of operations  in each layer of $\bx$, and $A$ is an adjacency matrix, representing the connection among these layers (i.e., network structure) in $\bx$. We can also view a neural network as a directed labeled graph where each layer is a node in a graph, and an operation in each layer is a node label (i.e., $A$ represents the graph structure, and $\Ss^{o}$ contains a set of node labels). We then propose to extract information from neural network architectures by distilling them into three separate quantities as follows: 

% \footnote{For a technical issue, when neural network does not contain any operation, its $n$-gram equal to a zero vector, we set its normalization as a uniform unit vector.}

\paragraph{$\bullet$  $\boldsymbol{n}$-gram representation for layer operations.} Each neural network consists of several operations from the input layer to the output layer. Inspired by the $n$-gram representation for a document in natural language processing, we view a neural network as a document and its operations as words. Therefore, we can use $n$-grams (i.e., $n$-length paths) to represent operations used in the neural network. We then normalize the $n$-gram, and denote it as $\bx^{o}$ for a neural network $\bx$.

Particularly, for $n=1$, the $n$-gram representation is a frequency vector of operations, used in Nasbot \cite{kandasamy2018neural}. 
When we use all $n \le \ell$ where $\ell$ is the number of network layers, the $n$-gram representation shares the same spirit as the path encoding, used in Bananas~\cite{white2019bananas}.

%Let denote $\SS$ be the set of $n$-gram operations, we can represent the  empirical measures as

%\[
%\SS^n = \underbrace{\SS \times \cdots \times \SS}_{\text{$n$ times}}
%\]

% In our fixed-support case (i.e., supports in $\SS^n$), we can use histograms, frequency for each $n$-gram operation, in $\RR^{\left| \SS^n \right|}$ as an alternative representation for these empirical measures.

Let $\SS$ be the set of operations, and $\SS^n = \SS \times \SS \times \dotsm \times \SS$ ($n$ times of $\SS$), the $n$-gram can be represented as empirical measures in the followings
\begin{align}
\omega^{o}_{\bx} = \sum_{s \in \SS^n} \bx^{o}_{s} \delta_{s},
\qquad \qquad
\omega^{o}_{\bz} = \sum_{s \in \SS^n} \bz^{o}_{s} \delta_{s},
\end{align}
where $\bx_s^o$ and $\bz_s^o$ are the frequency of $n$-gram operation $s \in \SS^n$ in architecture $\bx$ and  $\bz$, respectively. %Since the supports of $\omega_{\bx}^{o}, \omega_{\bz}^{o}$ are in $\SS^n$, we can use histograms in $\RR^{\left| \SS^n \right|}$ where each support $s \in \SS^n$ is a dimension of the histograms, and their values are $\bx_{s}^{o}, \bz_{s}^{o}$ respectively, as an alternative representation for these empirical measures. 

We can leverage the TW distance to compare the $n$-gram representations $\omega_{\bx}^{o}$ and $\omega_{\bz}^{o}$ using Eq. (\ref{equ:OT_LT}), denoted as $\dOperation \!(\omega_{\bx}^{o}, \omega_{\bz}^{o})$.  To compute this distance, we utilize a predefined tree structure for network operations by hierarchically grouping similar network operations into a tree as illustrated in Fig.~\ref{fg:TW}. We can utilize the domain knowledge to define the grouping and the edge weights, such as we can have $\texttt{conv1}$ and $\texttt{conv3}$ in the same group and $\texttt{maxpool}$ is from another group.  Inspired by the partition-based tree metric sampling \cite{le2019tree}, we define the edge weights decreasing when the edge is far from the root. Although such design can be subjective, the final distance (defined later in Eq. (\ref{equ:dNN})) will be calibrated and normalized properly when modeling with a GP in \S \ref{sec:NASBOW}. We refer to Fig. \ref{fig:example_TW101} and Appendix \S \ref{sec:app:TW_Ex} for the example of TW computation for neural network architectures.

  \begin{figure*}
    %\centering
    \includegraphics[width=0.95\linewidth]{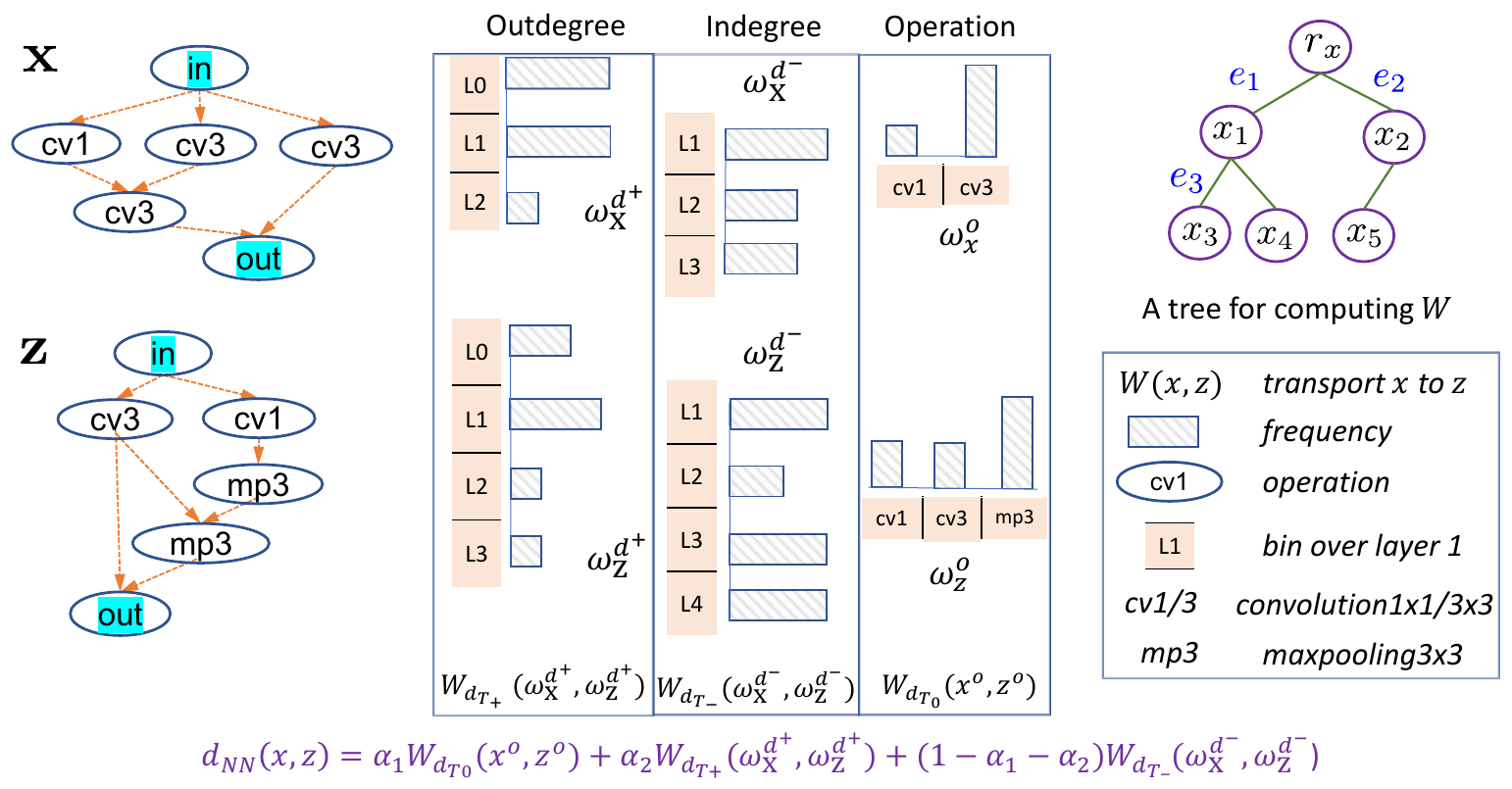}
     \vspace{-5pt}
    \caption{\label{fg:TW}
We represent two architectures $\bx$ and $\bz$ by network structure (via outdegree and indegree) and  network operation (using $1$-gram in this example). The similarity between each respective representation is estimated by tree-Wasserstein to compute the minimal cost of transporting one object to another. As a nice property of optimal transport, our tree-Wasserstein can handle \textit{different} layer sizes and \textit{different} operation types. The weights in each histogram are calculated from the architectures. The histogram bins in outdegree and indegree are aligned with the network structure in the left. See the Appendix \S \ref{sec:app:TW_Ex}  for detailed calculations.}%We will have three respective trees for these representations. 
%We can use the tree metric in \subref{fg:TW_TM} for $\dOperation$ on the set of network operations (cv: convolution, mp: max-pooling). Consequently, the tree metric will play a similar role as cost metric in \citep{kandasamy2018neural} (Table 1) for network operations. In case, one only uses $1$-gram (as in NASBOT\citep{kandasamy2018neural}), the representation of neural networks $\bx$, $\bz$ is the same, or $W(\bx^{o}, \bz^{o})=0$. However, for $n$-grams representation with $n \ge 2$, we have $W(\bx^{o}, \bz^{o})>0$. Moreover, although $\bx$ and $\bz$ have the same number of operation types, their structures are different, which are captured in $\dIn$ and $\dOut$ to distinguish $\bx$ and $\bz$.}
     \vspace{-5pt}

\end{figure*}

\paragraph{$\bullet$ Indegree and outdegree representations for network structure.} 

% vcom{temporarily remove to focus on the main point and for space}
% We can use the adjacency matrix $A$ directly to capture a network structure, and then may employ the Gromov-Wasserstein to compare those adjacency matrices. However, GW is not negative definite to build a p.s.d. kernel ~\cite{PeyreCuturiBook} (\S8.3), a required property for modeling with GP (see Lemma \ref{lem:GP_psd}). 

%By observing the common representation in neural network layers that 

We extract the \textit{indegree} and \textit{outdegree} of each layer, which are the number of ingoing and outgoing layers respectively, as an alternative way to represent a network structure.  We denote $L_{\bx}$ as the set of all layers which one can reach from the input layer for neural network $\bx$. Let $\eta_{x,\ell}$ and $M_x$ be lengths of the longest paths from an input layer to the layer $\ell$ and to the output layer respectively. Such paths interpret the order of layers in a neural network which starts with the input layer, connect with some middle layers, and end with the output layer, we represent the indegree and outdegree of network layers in $\bx$ as empirical measures $\omega^{d^-}_{\bx}$ and $\omega^{d^+}_{\bx}$, defined as 
\begin{align}
\hspace{-0.7 em}\omega^{d^-}_{\bx} = \sum_{\ell \in L_{\bx}} \bx^{d^{-}}_{\ell} \delta_{\frac{\eta_{x,\ell}+1}{M_x+1}}, \hspace{1.8 em} \omega^{d^+}_{\bx} = \sum_{\ell \in L_{\bx}} \bx^{d^{+}}_{\ell} \delta_{\frac{\eta_{x,\ell}+1}{M_x+1}},
\end{align}
where $\bx^{d^{-}}_{\ell}$ and $\bx^{d^{+}}_{\ell}$ are the normalized indegree and outdegree of the layer $\ell$ of $\bx$ respectively.

For indegree and outdegree information, the supports of empirical measures $\omega^{d^-}_{\bx}$, and $\omega^{d^-}_{\bz}$ are in one-dimensional space that a tree structure reduces to a chain of supports. Thus, we can use $\dIn \!\left( \omega^{d^-}_{\bx}, \omega^{d^-}_{\bz} \right)$ to compare those empirical measures.\footnote{Since the tree is a chain, the TW distance is equivalent to the univariate OT.} Similarly, we have $\dOut \!\left( \omega^{d^+}_{\bx}, \omega^{d^+}_{\bz} \right)$ for empirical measures $\omega^{d^+}_{\bx}$ and $\omega^{d^+}_{\bz}$ built from outdegree information.

\paragraph{Tree-Wasserstein distance for neural networks.} Given neural networks $\bx$ and $\bz$, we consider three separate TW distances for the $n$-gram, indegree and outdegree representations of the networks respectively. Then, we define $d_{\dNN}$ as a convex combination with nonnegative weights $\left\{\alpha_1, \alpha_2, \alpha_3 \mid \sum_i \alpha_i = 1, \alpha_i \ge 0\right\}$ for $\dOperation$, $\dIn$, and $\dOut$ respectively, to compare neural networks $\bx$ and $\bz$ as:
\begin{multline}\label{equ:dNN}
    d_{\text{NN}}(\bx, \bz) = \alpha_1 \dOperation (\bx^{o}, \bz^{o}) + \alpha_2 \dIn \left( \omega^{d^{-}}_{\bx}, \omega^{d^{-}}_{\bz} \right) \\
    + \left( 1- \alpha_1 -\alpha_2 \right) \dOut \left( \omega^{d^{+}}_{\bx}, \omega^{d^{+}}_{\bz} \right).
\end{multline}
The proposed discrepancy $d_{\dNN}$ can capture not only the frequency of layer operations, but also network structures, e.g., indegree and outdegree of network layers.

We illustrate our proposed TW for neural networks in Fig. \ref{fig:TW_distances} describing each component in Eq. (\ref{equ:dNN}). We also describe the detailed calculations in the Appendix \S \ref{sec:app:TW_Ex}. We highlight a useful property of our proposed $d_{\text{NN}}$: it can compare two architectures with \textit{different} layer sizes and/or operations sizes.

%\vcom{distance or measure @Answer: measure here is a kind of measurement (dissimilarity measure). Note the proposed one is not a distance, but only pseudo-distance --> so maybe we need more care when using "distance"}

\begin{proposition}\label{pro:dNN}
The $d_{\text{NN}}$ for neural networks is a pseudo-metric and negative definite. %\vcom{Thus, we can derive a p.s.d. kernel upon $d_{\dNN}$ @Answer: Imo, the main information of this proposition is about properties for $d_{\dNN}$, you may add the discussion about p.s.d. kernel in its following discussion paragraph.}.
\end{proposition}

Proof of Proposition \ref{pro:dNN} is placed in the Appendix \S\ref{sec:proof_pro:dNN}.

%, p.74),
Our discrepancy $d_{\dNN}$ is negative definite as opposed to the OT for neural networks considered in \citet{kandasamy2018neural} which is indefinite. Therefore, from Proposition~\ref{pro:dNN} and following Theorem 3.2.2 in \citet{berg1984harmonic},  we can derive a positive definite TW kernel upon $d_{\dNN}$ for neural networks $\bx, \bz$ as 
\begin{equation}\label{equ:TWKernel_NN}
k(\bx,\bz)=\exp\ \bigl(- d_{\dNN}(\bx, \bz)/\sigma^2_l \bigr),
\end{equation}
where the scalar $\sigma^2_l$ is the length-scale parameter. 
Our kernel has three hyperparameters including a length-scale $\sigma^2_l$ in Eq.~\eqref{equ:TWKernel_NN}; $\alpha_1$ and $\alpha_2$ in Eq.~\eqref{equ:dNN}. These hyperparameters will be estimated by maximizing the log marginal likelihood (see Appendix \S\ref{app:sec:opt_hyper}). We refer to the Appendix \S\ref{sec:app:distance_property} for a further discussion about the properties of the pseudo-distance $d_{\dNN}$.

\section{Neural Architecture Search with Gaussian Process and k-DPP} \label{sec:NASBOW}

%\subsection{Bayesian Optimization}\label{sec:BO}

\paragraph{Problem setting.}
We consider a noisy black-box function $f:\mathbb{R}^{d} \rightarrow \mathbb{R}$ over some domain $\mathcal{X}$ containing neural network architectures. As a black-box function, we do not have an explicit formulation for $f$ and it is expensive to evaluate. Our goal is to find the best architecture $\bx^* \in \mathcal{X}$ such that 
%\underset{x \in \mathcal{X}}
\begin{align}
    \bx^* = \argmax {\bx \in \mathcal{X}} \, f(\bx). 
    \label{Eq.maximizing_fx}
\end{align}

We view the black-box function $f$ as a machine learning experiment which takes an input as a neural network  architecture $\bx$ and produces an accuracy $y$. We can write $y=f(\bx)+\epsilon$ where we have considered Gaussian noise  $\epsilon \sim \mathcal{N}(0,\sigma^2_f)$ given the noise variance $\sigma^2_f$ estimated from the data.

Bayesian optimization (BO) optimizes the black-box function by sequentially evaluating it \cite{garnett2010bayesian,Shahriari_2016Taking,nguyen2019knowing}. Particularly, BO can speed up the optimization process by using a probabilistic model to guide the search \cite{Snoek_2012Practical}. BO has demonstrated impressive success for optimizing the expensive black-box functions across domains.% {\color{blue} As a black-box function, we have no closed-form expression for $f$ and it is expensive to evaluate @T: similar sentence as above}.

\paragraph{Surrogate models.}
Bayesian optimization reasons about $f$ by building a surrogate model, such as a Gaussian process (GP) \cite{Rasmussen_2006gaussian}, Bayesian deep learning \cite{Springenberg_2016Bayesian} or deep neural network \cite{Snoek_2015Scalable,white2019bananas}. Among these choices, GP is the most popular model, offering three key benefits: (i) closed-form uncertainty estimation, (ii) evaluation efficiency, and (iii) learning hyperparameters. GP imposes a normally distributed random variable at every point in the input space. The predictive distribution for a new observation also follows a Gaussian distribution \cite{Rasmussen_2006gaussian} where we can estimate the expected function value $\mu(\bx)$ and the predictive uncertainty $\sigma(\bx)$ as
\begin{align}
 \mu\left(\bx'\right)	&=\mathbf{k}(\bx',\bX) \left[ \bK + \sigma^2_f \idenmat \right]^{-1}\mathbf{y}
    \label{eq:GPmean} \\
    \sigma^{2}\left(\bx'\right)	&=k_{**}-\mathbf{k}(\bx',\bX)\left[ \bK + \sigma^2_f \idenmat \right]^{-1}\mathbf{k}^T(\bx',\bX) \label{eq:GPvar}
\end{align}

% \begin{align}
%     \mu\left(\bx'\right)	=\mathbf{k}(\bx',X) \left[ \mathbf{K}(X,X) + \sigma^2_n \idenmat \right]^{-1}\mathbf{y}
%     \label{eq:GPmean}
%     \end{align}
%     \begin{align}
%     \sigma^{2}\left(\bx'\right)	=k(\bx',\bx')-\mathbf{k}(\bx',X)\left[ \mathbf{K}(X,X) + \sigma^2_n \idenmat \right]^{-1}\mathbf{k}(\bx',X)^{T} \label{eq:GPvar}
% \end{align}
where  $\bX=[\bx_1,...\bx_N]$ and $\by=[y_1,..y_N]$ are the collected architectures and performances respectively; $K(U,V)$ is a covariance matrix whose element $(i,j)$ is calculated as $k(\bx_{i},\bx_{j})$ with $\bx_{i}\in U$ and $\bx_{j}\in V$; $k_{**}=k(\bx',\bx')$; $\bK:=\mathbf{K(X,X)}$; $\sigma^2_f$ is the measurement noise variance and $\idenmat$ is the identity matrix.

%$k_{i,j}$ is the $(i,j)$ element of covariance matrix $K$; $k^T$ is a transpose matrix of $\mathbf{k}$

\paragraph{Generating a pool of candidates $\mathcal{P}_t$.}
%We define the convolutional cell-based search space as used in \cite{ying2019bench,NASBench201}. 
We follow \citet{kandasamy2018neural,white2019bananas} to   generate a list of candidate networks using an evolutionary algorithm \cite{back1996evolutionary}. First, we stochastically select top-performing candidates with higher acquisition function values. Then, we apply a mutation operator to each candidate to produce modified architectures. Finally, we evaluate the acquisition given these mutations, add them to the initial pool, and repeat for several steps to get a pool of candidates $\mathcal{P}_t$. We design the ablation study in Fig. \ref{fig:ablation_pool} demonstrating that the evolution strategy  outperforms the random strategy for this task.
%  than those with lower values
%Particularly, we begin with an initial pool of networks and evaluate the acquisition $\alpha(\bx)$ on those networks. Then we generate a set of $N$ mutations of this pool as follows. 

\paragraph{Optimizing hyperparameters.}
We optimize the model hyperparameters by maximizing the log marginal likelihood. We present the derivatives for estimating the hyperparameters $\alpha_1$ and $\alpha_2$ of the tree-Wasserstein $d_{\text{NN}}$ for neural networks in the Appendix \S\ref{app:sec:opt_hyper}. We shall optimize these variables via multi-started gradient descent. 

% As opposed to NASBOT with 11 hyper-parameters of its own, our three hyperparameters are less vulnerable to overfitting when optimizing with the GP marginal likelihood.

\subsection{Sequential NAS using Bayesian optimization}
%Using the GP surrogate model above, we then construct an acquisition function $\alpha(\bx)$ to determine a next point to evaluate. 
We sequentially suggest a \textit{single} architecture for evaluation using a decision function $\alpha(\bx)$ (i.e., acquisition function) from the surrogate model. This acquisition function is carefully designed to trade off between exploration of the search space and exploitation of current promising regions. We utilize the GP-UCB \cite{Srinivas_2010Gaussian} as the main decision function $\alpha(\bx)=\mu(\bx)+\kappa\sigma(\bx)$ where $\kappa$ is the parameter controlling the exploration, $\mu$ and $\sigma$ are the GP predictive mean and variance in Eqs. (\ref{eq:GPmean}, \ref{eq:GPvar}). Empirically, we find that this GP-UCB generally performs better than expected improvement (EI) (see the Appendix \S \ref{sec:model_analysis}) and other acquisition functions (see \cite{white2019bananas}). We note that the GP-UCB also comes with a theoretical guarantee for convergence \cite{Srinivas_2010Gaussian}.

We maximize the acquisition function to select the next architecture $\bx_{t+1}=\arg\max_{\bx\in\mathcal{P}_t}\alpha_{t}\left(\bx\right)$. This maximization is done on the discrete set of candidate $\mathcal{P}_t$ obtained previously. The selected candidate is the one we expect to be the best if we are optimistic in the presence of uncertainty. %The selected candidate exhibits the highest potential of being the optimal architecture. 

%Compared to generic vector-input BO methods, our BO over architectures needs to find an architecture $\bx_t$ on $\mathcal{X}$, optimizing the acquisition function over a collection of architecture candidates.  %In this auxiliary maximisation problem, the acquisition function form is known and easy to evaluate and can be easily optimized by standard numerical toolboxes.

%\paragraph{Optimizing acquisition functions}
%Optimization an acquisition function  is one of the primary steps in BO to suggest the next architecture to evaluate. 

\begin{algorithm}[t]
	    \caption{Sequential and Parallel NAS using Gaussian process with tree-Wasserstein kernel}\label{alg:CoCaBO_batch_selection}
	\begin{algorithmic}[1]
		\vspace{0.5em}
		\STATE {\bfseries Input:} Initial data $\mathcal{D}_{0}$, black-box function $f(\bx)$.  {\bfseries Output:} The best architecture $\mathbf{x^*}$
	
        \FOR{$t=1, \ldots, T$}
        
        	%and build an acquisition function $\alpha$
        	\STATE Generate architecture candidates $\mathcal{P}_t$ by random permutation from the top architectures
    		\STATE Learn a GP (including hyperparameters) using TW from $\mathcal{D}_{t-1}$ to perform estimation over $\mathcal{P}_t$ including (i) covariance matrix $K_{\mathcal{P}_t}$, (ii) predictive mean $\mu_{ \mathcal{P}_t}$ and (iii) predictive variance $\sigma_{\mathcal{P}_t}$ 
        		
    		\STATE If \textbf{Sequential:} (i) select a next architecture $\mathbf{x}_t = \argmax {\forall \mathbf{x} \in \mathcal{P}_t}  \  \alpha (\mathbf{x} \mid \mu_{\mathcal{P}_t}, \sigma_{\mathcal{P}_t})$; then (ii) evaluate the new architecture $y_t = f(\mathbf{x}_t)$; after that, (iii) augment $\mathcal{D}_t \leftarrow \mathcal{D}_{t-1} \cup \left( \mathbf{x}_t, y_t \right)$.
    		
    		\STATE If \textbf{Parallel:} (i) select $B$ architectures $\mathbf{X}_t=\left[ \mathbf{x}_{t,1}, ,...\mathbf{x}_{t,B}\right] = \text{k-DPP}(\bK_{\mathcal{P}_t)}$ in Eq. (\ref{eq:kdpp_quality_gp}); then (ii) evaluate in parallel $Y_t = f(\mathbf{X}_t)$; after that, (iii) augment $\mathcal{D}_t \leftarrow \mathcal{D}_{t-1} \cup \left( \mathbf{X}_t, Y_t \right)$. 
    	\ENDFOR
	%\RETURN $\mathcal{B}_t$
	\end{algorithmic}
\end{algorithm}

\subsection{Parallel NAS using Quality k-DPP and GP}
The parallel setting speeds up the optimization process by selecting a \textit{batch} of architectures for parallel evaluations. We present the k-determinantal point process (k-DPP) with \textit{quality} to select from a discrete pool of candidate $\mathcal{P}_t$ for (i) high-performing and (ii) diverse architectures that cover the most information while avoiding redundancy. In addition,  diversity  is an important property for
not being stuck at a local optimal architecture.

% . Conditioning on sampling sets of fixed cardinality $k$, we obtains a k-DPP

% \footnote{To avoid the conflict in notation with the kernel, we then use $B$ as the batch size}
The DPP \cite{Kulesza_2012Determinantal} is an elegant probabilistic measure used to model negative correlations within a subset and hence promote  its diversity.  A k-determinantal point process (k-DPP) \cite{Kulesza_2011kDPP}  is a distribution over all subsets of a ground set $\mathcal{P}_t$ of cardinality k. It is determined by a positive semidefinite kernel $\bK_{\mathcal{P}_t}$.
Let $\bK_A$ be the submatrix of $\bK_{\mathcal{P}_t}$ consisting of the entries $\bK_{ij}$ with $i, j \in A \subseteq \mathcal{P}_t$. Then, the probability of observing $ A \subseteq P$ is proportional to $\det(\bK_A)$, % \centering
\begin{minipage}[t]{0.51\columnwidth}%
\vspace{-10pt}
\begin{align}
\hspace{-1 em} P(A \subseteq \mathcal{P}_t) \propto \det(\bK_A),
\end{align}
\end{minipage}%
\begin{minipage}[t]{0.48\columnwidth}%
\vspace{-10pt}
\begin{align}
\bK_{ij}=q_i \phi_i^T \phi_j q_j.\label{eq:kdpp_quality}
\end{align}
\end{minipage}

%The DPPs model aversion between items. The probability that items $i$ and $j$ are together included in $X$ is $P({i,j})=K_{ii}K_{jj} - \left( K_{ij} \right)^2$. This probability then decreases with similarity $K_{ij}$ between item $i$ and item $j$. This key aversion property makes DPPs useful for exploration purpose in NAS where we want to select architectures that cover the most information while avoiding redundancy.

\paragraph{k-DPP with quality.}
While the original idea of the k-DPP is to find a diverse subset, we can extend it to find a subset that is both diverse and high-quality. For this, we write a DPP kernel $k$ as a Gram matrix,
$\bK = \Phi ^T \Phi$, where the columns of $\Phi$ are vectors representing items in the set $S$. We now take this one step further, writing each column $\Phi$ as the product of a quality term $q_i \in \mathcal{R}^+$ and a vector of normalized diversity features $\phi_i$, $|| \phi_i|| = 1$. The entries of the kernel can now be written in Eq. (\ref{eq:kdpp_quality}).

% (While $D = N$ is sufficient to decompose any DPP, we keep $D$ arbitrary since in practice we may wish to use high-dimensional feature vectors.) 

As discussed in \cite{Kulesza_2012Determinantal}, this decomposition of $\bK$ has two main advantages. First, it implicitly enforces the
constraint that $\bK$ must be positive semidefinite, which can potentially simplify learning. Second, it allows us to independently model quality and diversity, and then combine them into a unified model. Particularly, we have
\[
P_K(A) \propto \left( \prod_{i\in A} q_i^2 \right) \det( \phi_i^T \phi_i),
\]
where the first term increases with the quality of the selected items, and the second term increases with the diversity of the selected items.  Without the quality component, we would get a very diverse set of architectures. However, we might fail to include the most
high-performance architectures in $\mathcal{P}_t$, focusing instead on low-quality outliers. By integrating the two models, we can achieve a more balanced result.

%Intuitively, the determinant of $K_Y$ is equal to the squared volume of the parallelepiped spanned by the vectors $q_i \phi_i$ for $i \in Y$. The magnitude of the vector representing item $i$ is $q_i$, and its direction is $\phi_i$.

   \begin{figure*} %[H]  
     \centering
    % trim={<left> <lower> <right> <upper>}
    
    %\vspace{2pt}
     \begin{subfigure}[t]{0.245\textwidth}
    \centering
        \includegraphics[width=\linewidth]{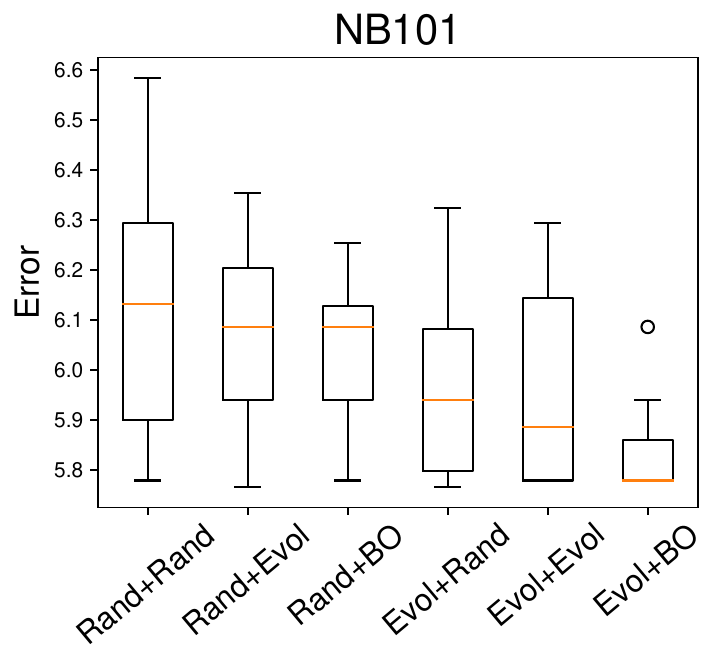} 
        %\caption{Corresponding time consumption of kernel matrices for $\TDA$ and document classification.} \label{fg:Time}
    \end{subfigure}
    \hfill
    \begin{subfigure}[t]{0.245\textwidth}
        \centering
        \includegraphics[width=\linewidth]{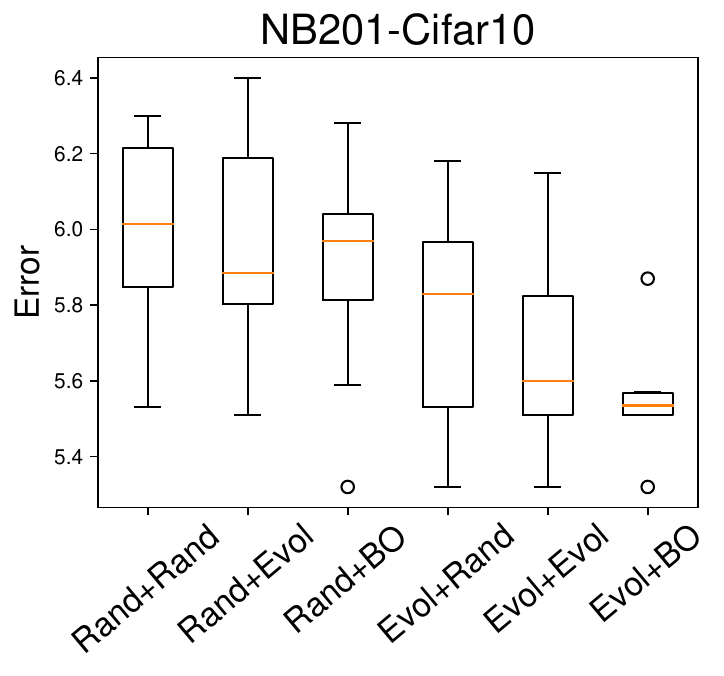} 
        %\caption{The KFDR graphs on granular packing system and SiO$_2$ datasets.} \label{fg:KFDR}
    \end{subfigure}
       % \vspace{2pt}
     \begin{subfigure}[t]{0.245\textwidth}
    \centering
        \includegraphics[width=\linewidth]{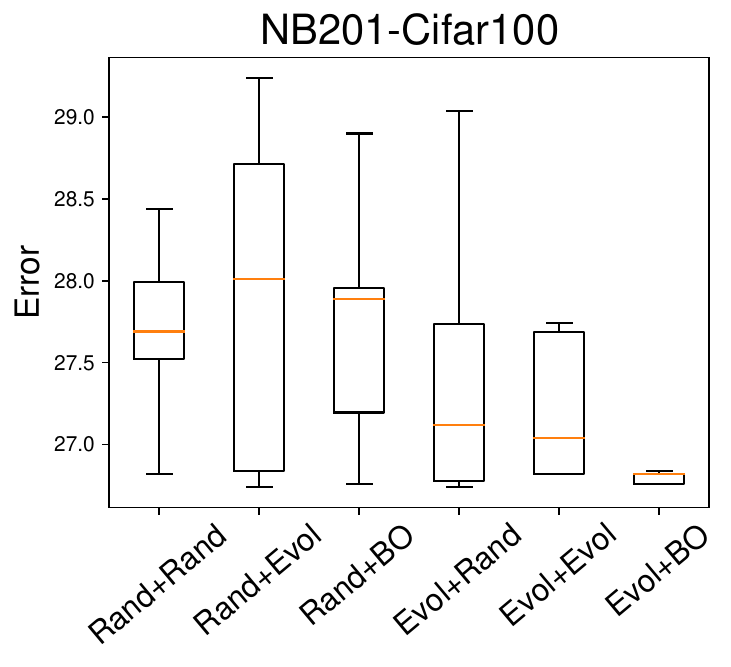} 
        %\caption{Corresponding time consumption of kernel matrices for $\TDA$ and document classification.} \label{fg:Time}
    \end{subfigure}
    \hfill
    \begin{subfigure}[t]{0.245\textwidth}
        \centering
        \includegraphics[width=\linewidth]{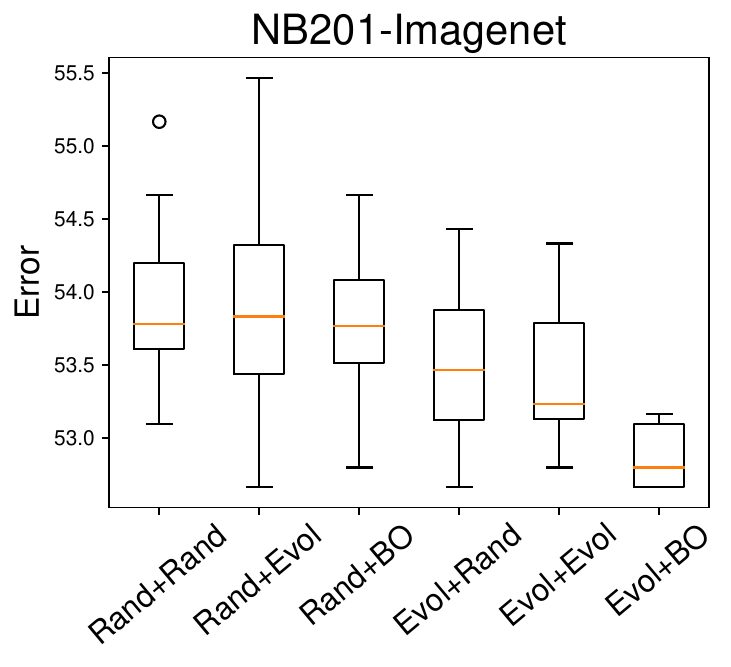} 
        %\caption{The KFDR graphs on granular packing system and SiO$_2$ datasets.} \label{fg:KFDR}
    \end{subfigure}

    \vspace{-10pt}
    \caption{We study the relative contribution of the strategies generating a pool of candidate $\mathcal{P}$ (Rand and Evolution) versus the main optimization algorithm (Rand, Evolution and BO). The result shows that Evolution can help to improve the performance than Rand in generating a pool of candidate $\mathcal{P}$. Given the same strategy for generating $\mathcal{P}$, BO is significantly better than Rand and Evolution for optimization. Evol+BO is the design in our approach that leads to the best performance.} \label{fig:ablation_pool}
    \vspace{-5pt}
\end{figure*}

\paragraph{Conditioning.}
%Suppose we want to condition a k-DPP on the inclusion of a particular set A. For $|A|+|B| = k$ we have

In the parallel setting, given the training data, we would like to select high quality and diverse architectures from a pool of candidate $\mathcal{P}_t$ described above. We shall condition on the training data in constructing the covariance matrix over the testing candidates from $\mathcal{P}_t$. We make the following proposition in connecting the k-DPP conditioning and GP uncertainty estimation. This view allows us to learn the covariance matrix using GP, such as we can maximize the GP marginal likelihood for learning the TW distance and kernel hyperparameters for k-DPP.

\begin{proposition} \label{pro:kDPP}
Conditioned on the training set, the probability of selecting new candidates from a pool $\mathcal{P}_t$ is equivalent to the determinant of the Gaussian process predictive covariance matrix.
\end{proposition}

Proof of Proposition \ref{pro:kDPP} is placed in the Appendix \S \ref{sec:proof_pro:kDPP}.

We can utilize  the GP predictive mean $\mu(\cdot)$ in Eq. (\ref{eq:GPmean}) to estimate the quality for any unknown architecture $q_i$ defined in Eq. (\ref{eq:kdpp_quality}). Then, we construct the covariance (kernel) matrix over the test candidates for selection by rewriting Eq. (\ref{eq:kdpp_quality}) as 
\begin{align}
%\centering
\hspace{-0.8 em} \bK_{\mathcal{P}_t} \!\left( \bx_i,\bx_j \right) \!=\! \exp \bigl( - \mu(\bx_i) \bigr) \sigma(\bx_i, \bx_j) \exp \bigl( - \mu(\bx_j) \bigr), \label{eq:kdpp_quality_gp}
\end{align}
for all $\bx_i,\bx_j \in \mathcal{P}_t$  where $\mu(\bx_i)$ and $\sigma(\bx_i,\bx_j)$ are the GP predictive mean and variance defined in Eqs. (\ref{eq:GPmean}, \ref{eq:GPvar}). Finally, we  sample $B$ architectures from the covariance matrix $\bK_{\mathcal{P}_t}$ which encodes both the diversity (exploration) and high-utility (exploitation). The sampling algorithm requires precomputing the eigenvalues \cite{Kulesza_2011kDPP}. Sampling from a k-DPP requires $\mathcal{O}(N B^2)$ time overall where $B$ is the batch size.
% Each term in Eq. (\ref{eq:kdpp_quality_gp}) are naturally in range $[0,1]$, thus it balances between diversity and quality.
\paragraph{Advantages.}
The connection between GP and k-DPP allows us to directly sample diverse and high-quality samples from the GP posterior. This leads to the key advantage that we can \textit{optimally} sample a batch of candidates without the need of greedy selection. On the other hand, the existing batch BO approaches  rely  either on greedy strategy  \cite{Contal_2013Parallel,Desautels_2014Parallelizing,Gonzalez_2015Batch}  to sequentially select the points in a batch or independent sampling \cite{falkner2018bohb,Hernandez_2017Parallel}. The greedy algorithm is non-optimal and the independent sampling approaches can not fully utilize the information across points in a batch. 

We note that our k-DPP above is related to \cite{Kathuria_NIPS2016Batched}, but different from two perspectives that \citet{Kathuria_NIPS2016Batched} considers k-DPP for batch BO (i) in the continuous setting and (ii) using pure exploration (without quality). We will consider this as the baseline in our experiments.

% \begin{figure*} %[t]  
%      \centering
%     % trim={<left> <lower> <right> <upper>}
%     \includegraphics[trim=0cm 0.cm 0cm  0.cm, clip, width=0.24\linewidth]{fig/ablation_pool_nb101.pdf}
%         \includegraphics[trim=0cm 0.cm 0cm  0.cm, clip, width=0.24\linewidth]{fig/ablation_pool_cf10.pdf}
%     \includegraphics[trim=0cm 0.cm 0cm  0.cm, clip, width=0.24\linewidth]{fig/ablation_pool_cf100.pdf}
%         \includegraphics[trim=0cm 0.cm 0cm  0.cm, clip, width=0.24\linewidth]{AISTATS2021/fig/ablation_pool_imagenet.pdf}

%     \caption{We study the relative contribution of the strategies generating a pool of candidate $\mathcal{P}$ (Rand and Evolution) versus the main optimization algorithm (Rand, Evolution and BO). The result shows that Evolution can help to improve the performance than Rand in generating a pool of candidate $\mathcal{P}$. Given the same strategy for generating $\mathcal{P}$, BO is significantly better than Rand and Evolution for optimization. Evol+BO is the design in our approach that leads to the best performance.} \label{fig:ablation_pool}
%     \vspace{-5pt}
% \end{figure*}

\section{Experiments}
\label{sec:Experiments}
We evaluate our proposed approach on both sequential and parallel neural architecture search (NAS).

\textbf{Experimental settings.} All experimental results are averaged over 30 independent runs with different random seeds. We set the number of candidate architecture $|\mathcal{P}_t|=100$. We utilize  the popular NAS  tabular datasets of Nasbench101 (NB101) \cite{ying2019bench} and Nasbench201 (NB201) \cite{NASBench201} for evaluations. TW and TW-2G stand for our TW using 1-gram and 2-gram representation, respectively. We 
release the Python code for our experiments at \url{https://github.com/ntienvu/TW_NAS}.

\subsection{Sequential NAS}

\textbf{Ablation study: different mechanisms for generating a pool of candidates $\mathcal{P}$.} We analyze the relative contribution of the process of generating architecture candidates versus the main optimization algorithm in Fig. \ref{fig:ablation_pool}. The result suggests that the evolutionary algorithm is better than a random strategy to generate a pool of candidates $\mathcal{P}$. Given this generated candidate set $\mathcal{P}$, BO is significantly better than Rand and Evolution approaches. Briefly, the combination of Evol+BO performs the best across datasets.

\begin{figure*} [t]  
     \centering
    % trim={<left> <lower> <right> <upper>}
        \includegraphics[trim=0cm 0.cm 0cm  0.cm, clip, width=0.325\linewidth]{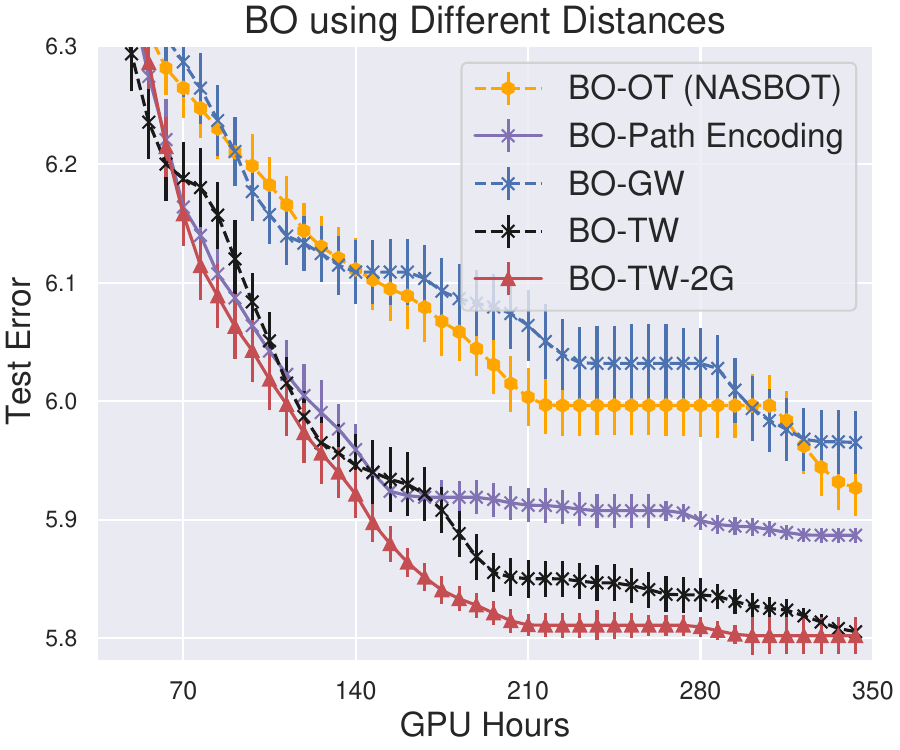}
    \includegraphics[trim=0cm 0.cm 0cm  0.cm, clip, width=0.325\linewidth]{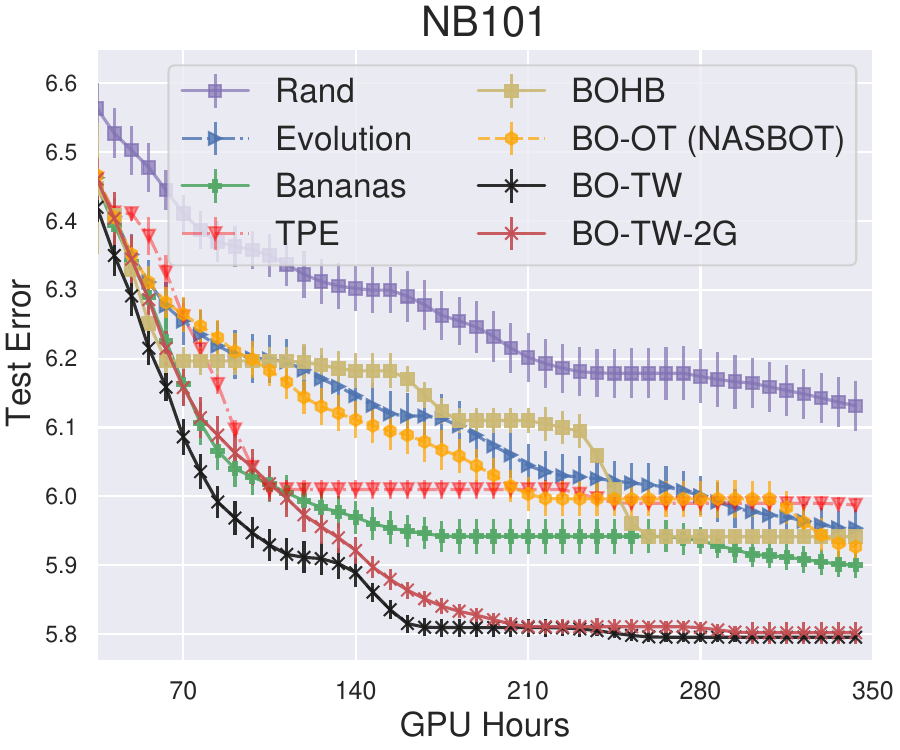}
        \includegraphics[trim=0cm 0.cm 0cm  0.cm, clip, width=0.325\linewidth]{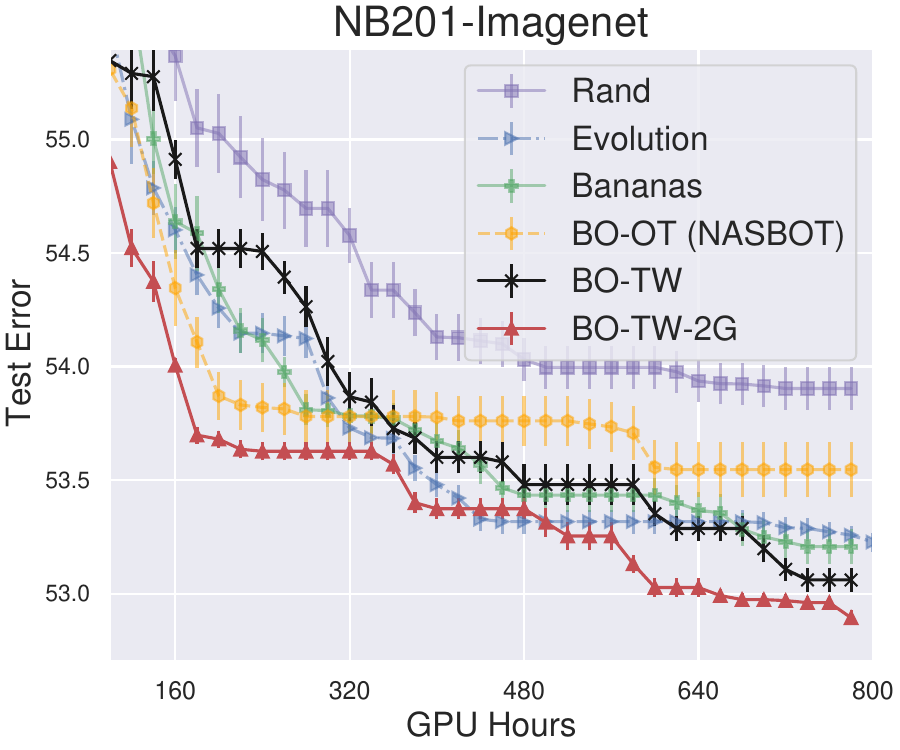}
    \vspace{-5pt}
    \caption{Sequential NAS on different distances for BO (\textit{left}) and different baselines (\textit{middle} and \textit{right}). Our approaches of BO-TW (\textbf{black} curve) and BO-TW 2G (\textcolor{red}{red} curve) for $1$-gram and $2$-gram representation consistently outperform the other baselines. We use $500$ iterations on NB101 and $200$ iterations on NB201.} \label{fig:NASBENCH101}
    \vspace{-5pt}
\end{figure*}

%We have demonstrated that BO is the better choice of searching for the best architecture against the different baselines in random search, evolutionary and other state-of-the-art methods. 
\textbf{Ablation study: different distances for BO.} We design an ablation study using different distances within a BO framework. Notably, we consider the vanilla optimal transport (Wasserstein distance) in which we follow \citet{kandasamy2018neural} to define the cost metric for OT. This baseline can be seen as the modified version of the Nasbot \cite{kandasamy2018neural}. In addition, we compare our approach with the BO using the Gromov-Wasserstein distance \cite{memoli2011gromov} (BO-GW) and path encoding (BO-Path Encoding) as used in \cite{white2019bananas}. The results in the left plot of Fig. \ref{fig:NASBENCH101} suggest that the proposed TW using 2-gram performs the best among the BO distance for NAS.
The standard OT and GW will result in (non-p.s.d.) indefinite kernels. For using OT and GW in our GP, we keep adding (``jitter'') noise to the diagonal of the kernel matrices until they become p.s.d. kernels. We make use of the POT library \cite{flamary2017pot} for the implementation of OT and GW. %We note that the optimal transport approaches, such as OT, GW, TW can scale well to more number of nodes, layers while the path-based encoding is limited to.

%\vcom{comment on n-gram, why dont we use all n-gram @Answer: In theory, you can consider n-gram as the extreme case in BANANAS. No problem. In case, $2$-gram may be enough, we may not need to go more. This is a similar question as the n-gram in NLP, when n grows, the number of candidates also grows exponentially, and in practice, we may have very sparse n-gram representation (and maybe some features from n-gram with big n are not so meaningful any more.).}
While our framework is able to handle $n$-gram representation, we learn that $2$-gram is empirically the best choice.  This choice is well supported by the fact that two convolution layers of $3\times3$ stay together can be used to represent for a special effect of $5\times5$ convolution kernel. In addition, the use of full $n$-gram may result in very sparse representation and some features are not so meaningful anymore. Therefore, in the experiment we only consider $1$-gram and $2$-gram. %This property can not be captured by a traditional $1$-gram representation. 

\textbf{Sequential NAS.} We  validate our GP-BO model using tree-Wasserstein on the sequential setting.  Since NB101 is somewhat harder than NB201, we allocate $500$ queries for NB101 and $200$ queries for NB201 including $10\%$ of random selection at the beginning of BO.

We compare our approach against common baselines including Random search, evolutionary search, TPE \cite{Bergstra_2011Algorithms}, BOHB \cite{falkner2018bohb}, Nasbot \cite{kandasamy2018neural} and Bananas \cite{white2019bananas}. We use the AutoML library for TPE and BOHB\footnote{\url{https://github.com/automl/nas_benchmarks}} including the results for NB101, but not NB201.  We do not compare with Reinforcement Learning approaches \cite{pham2018efficient} and AlphaX \cite{alphax} which have been shown to perform poorly in \cite{white2019bananas}. %It requires 3000 GPU days to train as shown in [].

\begin{figure} [t]%{r}{0.4\textwidth}
    %\vspace{-30pt}
       \vspace{-15pt}
  \begin{center}
    \includegraphics[width=0.44\textwidth]{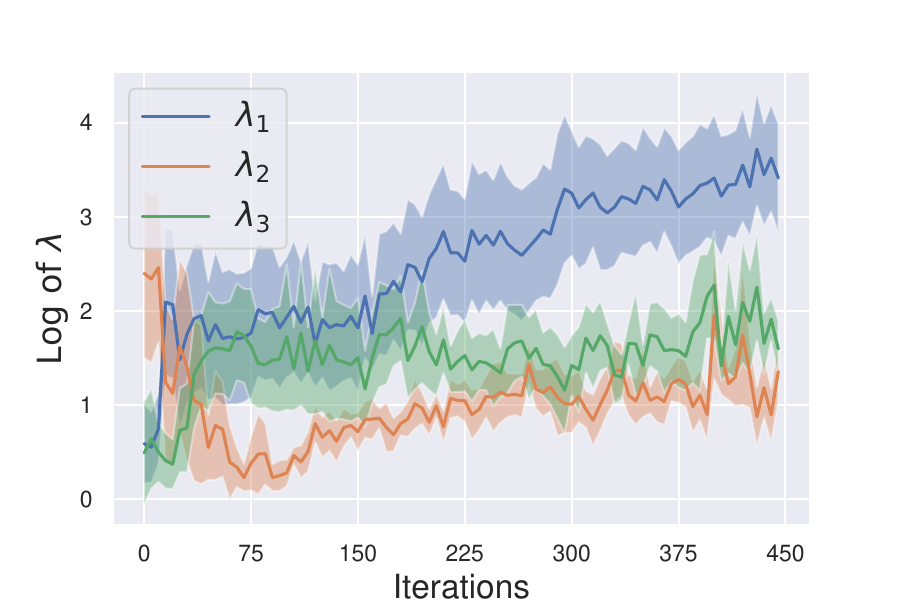}
  \end{center}
      \vspace{-10pt}
    \caption{Estimated hyperparameters on NB101.} \label{fig:estimated_lambdas}
    %\vspace{-35pt}
    \vspace{-5pt}
\end{figure}

We show in Fig. \ref{fig:NASBENCH101} that our tree-Wasserstein including 1-gram and 2-gram will result in the best performance with a wide margin to the second best -- the Bananas \cite{white2019bananas}, which needs to specify a meta neural network with extra hyperparameters (layers, nodes, learning rate). Random search performs poorly in NAS due to the high-dimensional and complex space. Our GP-based optimizer offers a closed-form uncertainty estimation without iterative approximation in neural network (via back-propagation). As a property of GP, our BO-TW can generalize well using fewer observations. This can be seen in the right plot of Fig. \ref{fig:NASBENCH101} that our approaches can outperform Bananas when the number of BO iteration (or the number of network architectures for training) is small. On the other hand, both Bananas and ours are converging to the same performance when the training data becomes abundant -- but this is not the case in practice for NAS.

%We have presented the performance comparison on the test set. We have additional results on comparison on the validation set in the appendix.
%Most NAS methods use the validation set to evaluate architectures after the architecture is optimized on the training set. The validation performance of the architectures serves as supervision signals to update the searching algorithm. The test set is to evaluate the performance of each searching algorithm by comparing the indicators (e.g., accuracy, model size, speed) of their selected architectures.

% Bohb \cite{falkner2018bohb},

\begin{figure*} [t]  
     \centering
    % trim={<left> <lower> <right> <upper>}
    \includegraphics[trim=0cm 0.cm 0cm  0.cm, clip, width=0.325\linewidth]{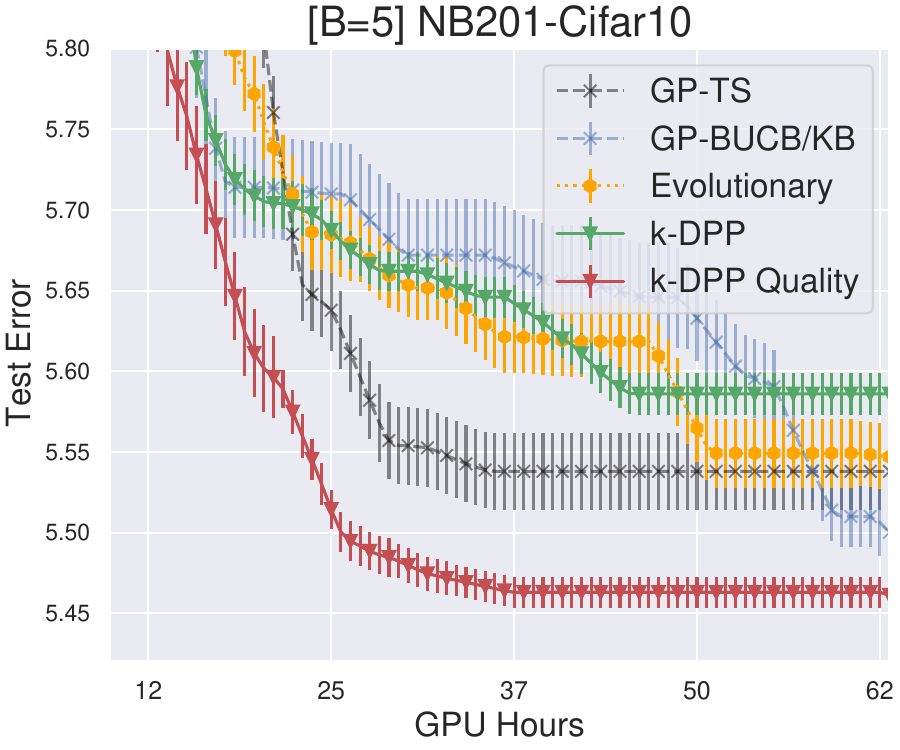}
  \includegraphics[trim=0cm 0.cm 0cm  0.cm, clip, width=0.325\linewidth]{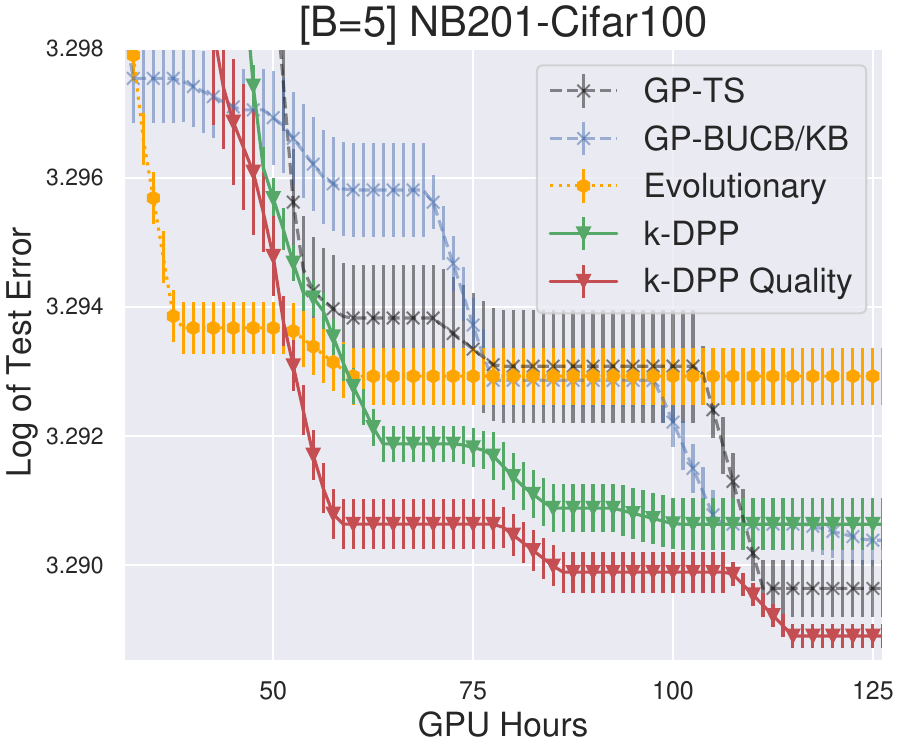}
 \includegraphics[trim=0cm 0.cm 0cm  0.cm, clip, width=0.325\linewidth]{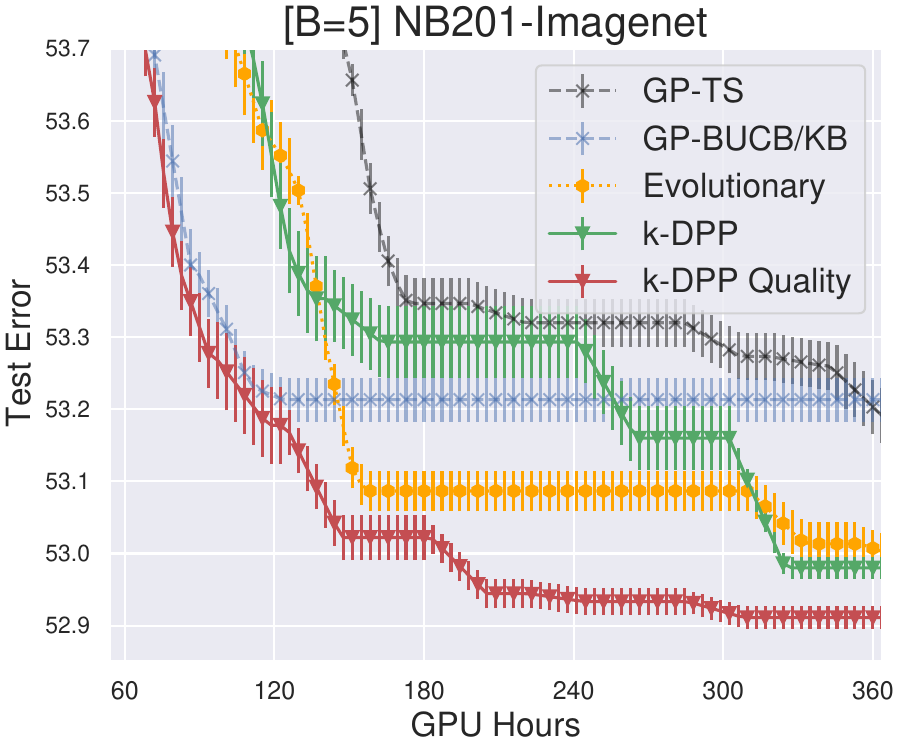}
     \vspace{-3pt}
    \caption{Batch NAS comparison using TW-2Gram and a batch size $B=5$. Our proposed k-DPP Quality (\textcolor{red}{red}) outperforms other baselines in all cases, especially when the number of training architecture (iterations) is low. This is the desirable property of NAS when the training cost is extremely expensive. The experiments are run over $100$ iterations.} \label{fig:batch_NAS}
    \vspace{-3pt}
\end{figure*}

% [trim=0cm 0.cm 0cm  0.cm, clip, width=0.4\linewidth]

\textbf{Estimating hyperparameters.} We plot the estimated hyperparameters 
\[
\lambda_1=\frac{\alpha_1}{ \sigma^2_l}, \quad \lambda_2=\frac{\alpha_2}{ \sigma^2_l}, \quad \lambda_3=\frac{1-\alpha_1-\alpha_2}{ \sigma^2_l},
\]
over iterations in Fig. \ref{fig:estimated_lambdas}. This indicates the relative contribution of the operation, indegree and outdegree toward the TW $d_{\text{NN}}$ for neural networks in Eq. (\ref{equ:dNN}). Particularly, the operation contributes receives higher weight and is useful information than either the individual indegree or outdegree.

\subsection{Parallel NAS}

%\paragraph{Batch NAS.}

We next demonstrate our model on selecting multiple architectures for parallel evaluation (i.e., parallel NAS) setting. There are fewer approaches for parallel NAS compared to the sequential setting. We select to compare our k-DPP quality against Thompson sampling \cite{Hernandez_2017Parallel},  GP-BUCB \cite{Desautels_2014Parallelizing} and k-DPP for batch BO \cite{Kathuria_NIPS2016Batched}. The GP-BUCB is equivalent to Kriging believer \cite{Ginsbourger_2010Kriging} when the hallucinated observation value is set to the GP predictive mean. Therefore, we label them as GP-BUCB/KB. We also compare with the vanilla k-DPP (without using quality) \cite{Kathuria_NIPS2016Batched}.

% and $90$ iterations of batch size $B=5$
We allocate a maximum budget of $500$ queries including $50$ random initial architectures. The result in Fig. \ref{fig:batch_NAS} shows that our proposed k-DPP quality is the best among the considered approaches. We refer to the Appendix for additional experiments including varying batch sizes and more results on NB201.

Our sampling from k-DPP quality is advantageous against the existing batch BO approaches \cite{Ginsbourger_2010Kriging,Desautels_2014Parallelizing,Kathuria_NIPS2016Batched,Hernandez_2017Parallel} in which we can optimally select a batch of architectures without relying on the greedy selection strategy. In addition, our k-DPP quality can leverage the benefit of the GP in estimating the hyperparameters for the covariance matrix.

\begin{figure} [t]  
     \centering
    % % trim={<left> <lower> <right> <upper>}
    \includegraphics[trim=0cm 0.cm 0cm  0.cm, clip, width=0.49\linewidth]{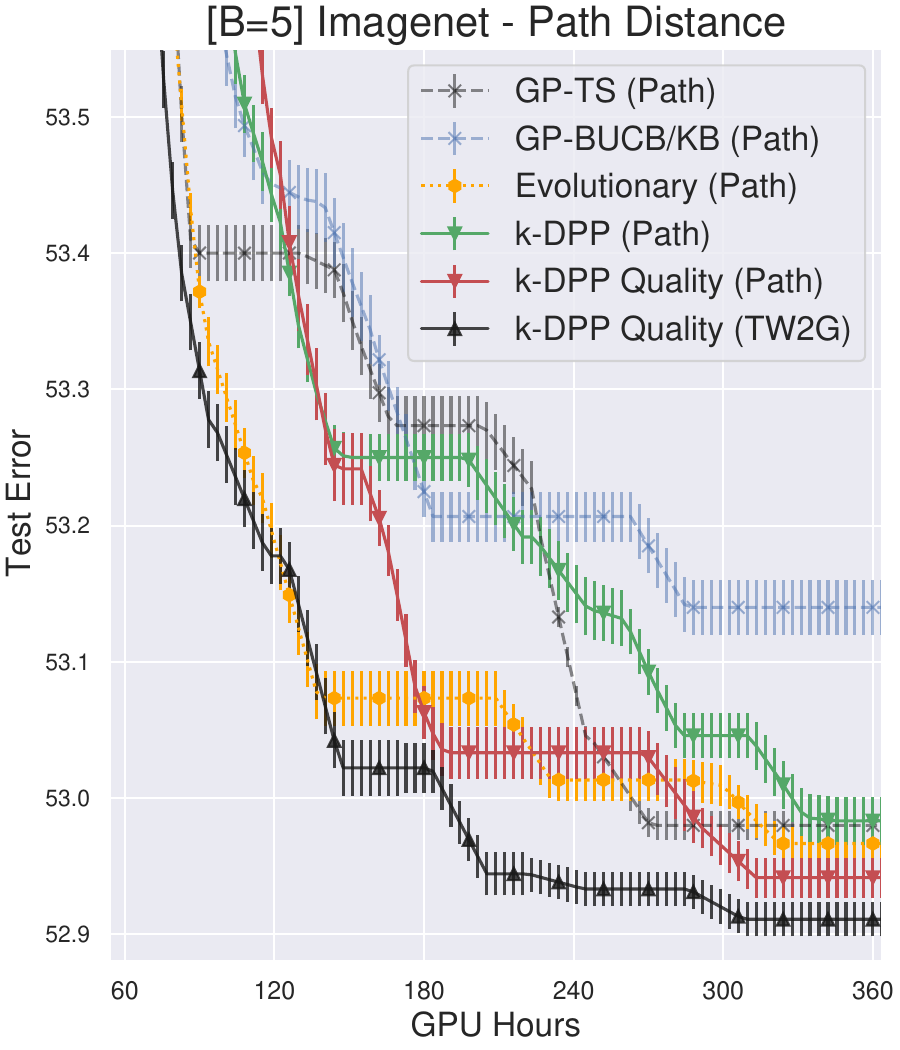}
    %\hspace{5pt}
        \includegraphics[trim=0cm 0.cm 0cm  0.cm, clip, width=0.49\linewidth]{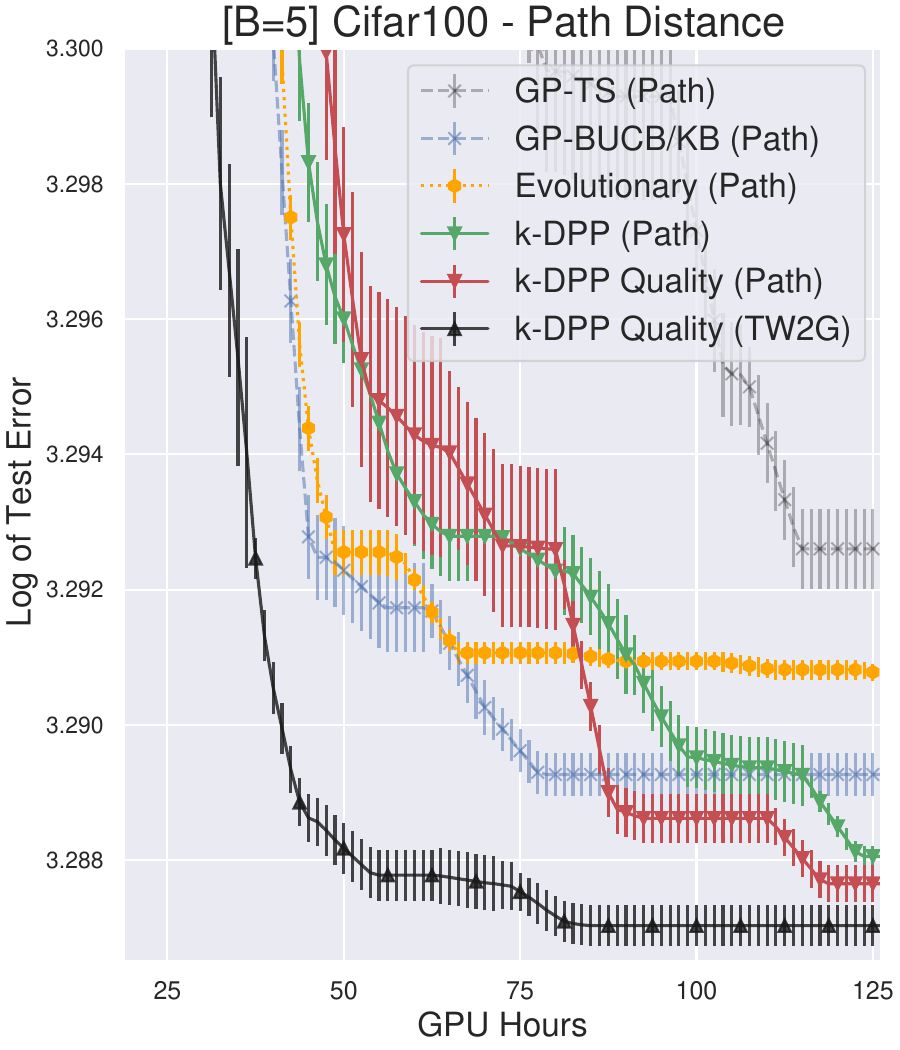}

    %  \begin{subfigure}[t]{0.23\textwidth}
    % \centering
    %     \includegraphics[width=\linewidth]{fig/batch_test_NB201_Imagenet_Path.pdf} 
    %     %\caption{Corresponding time consumption of kernel matrices for $\TDA$ and document classification.} \label{fg:Time}
    % \end{subfigure}
    % \hfill
    % \begin{subfigure}[t]{0.23\textwidth}
    %     \centering
    %     \includegraphics[width=\linewidth]{fig/batch_test_NB201_CF100_path.pdf} 
    %     %\caption{The KFDR graphs on granular packing system and SiO$_2$ datasets.} \label{fg:KFDR}
    % \end{subfigure}
        \vspace{-3pt}
    \caption{We compare different batch approaches using a path distance \cite{white2019bananas} and a batch size $B=5$ . We show that (i) the k-DPP quality outperforms the other batch approaches and (ii) the k-DPP using TW2G (a black curve) performs better than using path distance (a \textcolor{red}{red} curve).} \label{fig:batch_different_distance}
    \vspace{-3pt}
\end{figure}

\textbf{Ablation study of k-DPP quality with path distance.} In addition to the proposed tree-Wasserstein, we demonstrate the proposed k-DPP quality using path distance \cite{white2019bananas}. We show that our k-DPP quality is not restricted to TW-2G, but it can be generally used with different choices of kernel distances.

Particularly, we present in Fig. \ref{fig:batch_different_distance} the comparison using two datasets: Imagenet and Cifar100 in NB201. The results validate two following messages. First, our k-DPP quality is the best among the other baselines in selecting high-performing and diverse architectures. Second, our k-DPP quality with TW2G (a black curve) performs better than k-DPP quality using Path distance (a red curve). This demonstrates the key benefits of comparing two complex architectures as logistical supplier and receiver.

\section{Conclusion}
\label{sec:conclusion}

We have presented a new framework for sequential and parallel NAS. Our proposed framework constructs the similarity between architectures using tree-Wasserstein geometry. Then, it utilizes the Gaussian process surrogate for modeling and optimization. We draw the connection between GP predictive distribution to k-DPP quality for selecting diverse and high-performing architectures from discrete set. We demonstrate our model using Nasbench101 and Nasbench201 that our methods outperform the existing baselines in sequential and parallel settings.

%Future work will be twofolds. First, we can investigate multi-objective NAS to define and optimize a dual-objective function of loss and number of model parameters. Second, we can extend the current model to handle multi-fidelity across different of trained epochs.

% \begin{comment}
% \section*{Broader Impact}

% This paper presents a new machine learning approach that may have several societal impacts:

% \begin{itemize}
%     \item Our proposed technique in sequential and batch neural architecture search will be widely applicable to all deep neural networks - based approaches, including a wide range of machine learning algorithms (deep reinforcement learning, deep generative model, deep learning) and a tremendous range of applications (computer vision, natural language processing, manufacturing and more). 
%     \item Our NAS approaches automate the neural network design process that significantly saves cost and time for machine learning practitioners by taking them out of the tuning loop.
%     \item The tree-Wasserstein distance for neural networks can be of independent interest for different tasks in which utilizing the kernel or covariance matrix, such as neural network compression, neural network training and continual learning.
%     \item The k-DPP quality with Gaussian process can be of independent interest for other tasks such as selecting a diverse and high-quality set of points from a discrete set in document summarization, video summarization... etc.
    
% \end{itemize}
% \end{comment}
% -- this process is typically very expensive in time and cost

\section*{Acknowledgements}
%The authors also thank the anonymous ICML reviewers and the area chair for improving the paper.

We thank anonymous reviewers and area chairs for their comments to improve the paper. VN would like to thank NVIDIA for sponsoring GPU hardware and Google Cloud Platform for sponsoring computing resources in this project. TL acknowledges the support of JSPS KAKENHI Grant
number 20K19873. 

%\balance
\bibliographystyle{icml2021}
\bibliography{MLREF,vunguyen}

%%%%%%%%%%%%%%%%%%%%%%%%%%%%%%%%%%%%%%%%%%%%%%%%%%%%%%%%%%%%%%%%%%%%%%%%%%%%%%%%%

% \begin{center}
% \textbf{\Large{Appendix to ``Optimal Transport Kernels for Sequential and Parallel Neural Architecture Search"}}
% \end{center}

\newpage
\appendix

\onecolumn

\begin{center}
\textbf{\Large{\textit{Supplementary Material for:} Optimal Transport Kernels for Sequential and Parallel Neural Architecture Search}}
\end{center}

%\maketitle

In this supplementary, we first review the related works in sequential and batch neural architecture search in \S\ref{sec:app:related_work_NAS} and \S\ref{sec:app:batch} respectively. We next give further detailed information about the datasets in our experiments in \S\ref{sec:app:datasets}. Then, we derive detailed proofs for the results in the main text in \S\ref{sec:app:proofs}. After that, we give a brief discussion about the tree-Wasserstein geometry in \S\ref{sec:app:TW}, and follow with an example of TW computation for neural network architectures in \S\ref{sec:app:TW_Ex}. In \S\ref{app:sec:opt_hyper}, we provide additional details of the Bayesian optimization in estimating the hyperparameters. In \S\ref{sec:app:distance_property}, we show distance properties comparison. Finally, we illustrate further empirical comparisons and analysis for the model in \S\ref{sec:app:add_experiments}.

%%%%%%%%%%%%%%%%%%
\section{Related works in Neural architecture search} \label{sec:app:related_work_NAS}

%Reinforcement learning: This research area was revitalized when NASNet gained significant attention \cite{zoph2016neural}. NASNet is a reinforcement learning algorithm for NAS which achieves state-of-the-art results on CIFAR-10 and PTB; however, the algorithm requires 3000 GPU days to train. NASNet spurred a number of follow-up work, eventually pushing the required computation for NAS on CIFAR-10 and PTB to under 10 GPU days (Pham et al., 2018; Liu et al., 2018a;b; Ying et al., 2019; Kandasamy et al., 2018; Jin et al., 2018; Shah et al., 2018).

We refer the interested readers to the survey paper \cite{elsken2019neural} and the literature of neural architecture search\footnote{\url{https://www.automl.org/automl/literature-on-neural-architecture-search}} for the comprehensive survey on neural architecture search. Many different search strategies have been attempted to explore the space of neural architectures, including random search, evolutionary methods, reinforcement learning (RL), gradient-based methods and Bayesian optimization.

\paragraph{Evolutionary approaches.}
 \citet{real2017large,real2019regularized,suganuma2017genetic,liu2018hierarchical,shah2018amoebanet,xie2017genetic,elsken2019efficient}  have been extensively used for NAS. In the context of evolutionary, the mutation operations include adding a layer, removing a layer or changing the type of a layer (e.g., from convolution to pooling) from the neural network architecture. Then, the evolutionary approaches will update the population, e.g., tournament selection by removing the worst or oldest individual from a population.

\paragraph{Reinforcement learning.}
NASNet \cite{zoph2016neural} is a reinforcement learning algorithm for NAS which achieves state-of-the-art results on CIFAR-10 and PTB; however, the algorithm requires $3000$ GPU days to train. Efficient Neural Architecture Search (ENAS) \cite{pham2018efficient} proposes to use a controller that discovers architectures by learning to search for an optimal subgraph within a large graph. The controller is trained with policy gradient to select a subgraph that maximizes the validation set's expected reward. The model corresponding to the subgraph is trained to minimize a canonical cross-entropy loss. Multiple child models share parameters, ENAS requires fewer GPU-hours than other approaches and $1000$-fold less than "standard" NAS. Other reinforcement learning approaches for NAS have also proposed, such as MetaQNN \cite{baker2016designing} and BlockQNN \cite{zhong2018practical}.

%Many neuro-evolutionary approaches since then \cite{real2017large,real2019regularized,stanley2002evolving} (Angeline et al., 1994; Stanley and Miikkulainen, 2002; Stanley et al., 2009) use genetic algorithms to optimize both the neural architecture and its weights. However, when scaling to contemporary neural architectures with millions of weights for supervised learning tasks, SGD-based weight optimization methods currently outperform evolutionary ones.

%focus on how to change the landscape of the search space from a discrete to a differentiable one. 
\paragraph{Gradient-based approaches.}
 The works presented in \citet{luo2018neural,liu2019darts,dong2019searching,yao2020efficient}  represent the search space as a directed acyclic graph (DAG) containing billions of sub-graphs, each of which indicates a kind of neural architecture. To avoid traversing all the possibilities of the sub-graphs, they develop a differentiable sampler over the DAG.  The benefit of such an idea is that a differentiable space enables computation of gradient information, which could speed up the convergence of underneath optimization algorithm. Various techniques have been proposed, e.g., DARTS \cite{liu2019darts} %smooths design choices with softmax and trains an ensemble of networks
 , SNAS \cite{xie2018snas} 
 %enhances reinforcement learning with a smooth sampling scheme. 
 , and NAO \cite{luo2018neural}.
 %maps the search space into a new differentiable space with an auto-encoder. 
 While these approaches based on gradient-based learning can reduce the computational resources required for NAS, it is currently not well understood if an initial bias in exploring certain parts of the search space more than others might lead to the bias and thus result in premature convergence of NAS \cite{sciuto2019evaluating}. In addition, the gradient-based approach may be less appropriate for exploring different space (e.g., with a completely different number of layers), as opposed to the approach presented in this paper.
 
 % which biases they introduce into the search if the sampling distribution of architectures is optimized along with the one-shot model instead of fixing it. For instance, 
 %In general, a more systematic analysis of biases introduced by different performance estimators would be a desirable direction for future work.

% A recent algorithm, AlphaX, uses a meta neural network to perform NAS \cite{wang2019alphax}. This is a similar high-level idea to our work but with many key differences. Their search is progressive, and
% each iteration makes a small change to the current neural network, rather than choosing a completely new neural network. Therefore, their meta neural network is trained to optimize the selection of new
% actions as well as to predict the performance of the new progressive network. Furthermore, they use an adjacency matrix featurization for the inputs to the meta neural network, which we find to perform
% worse than our path-based encoding.

\paragraph{Bayesian optimization.}

%While a full comparison is lacking, there is preliminary evidence that BO approaches can outperform evolutionary algorithms for NAS. \cite{klein2018towards}.

BO  has been an emerging technique for black-box optimization when function evaluations are expensive \cite{Rana_ICML2017High,frazier2018tutorial,gopakumar2018algorithmic_NIPS, wan2021think}, and it has seen great success in hyperparameter optimization for deep learning \cite{li2018hyperband,boil,cocabo,pb2}. Recently, Bayesian optimization has been used for searching the best neural architecture \cite{kandasamy2018neural,jin2018auto,white2019bananas}. BO relies on a covariance function to represent the similarity between two data points. For such similarity representation, we can  (1) directly measure the similarity of the networks by optimal transport, then modeling with GP surrogate in \citet{kandasamy2018neural}; or (2) measure the graphs based on the path-based encodings, then modeling with neural network surrogate in \citet{white2019bananas}. OTMANN \cite{kandasamy2018neural} shares similarities with Wasserstein (earth mover’s) distances which also have an OT formulation. However, it is not a Wasserstein distance itself—in particular, the supports of the masses and the cost matrices change depending on the two networks being compared. One of the drawback of OTMANN is that it may not be negative definite for a p.s.d. kernel which is an important requirement for modeling with GP. This is the motivation for our proposed tree-Wasserstein.

%Drawback of NASBOT [to be filled in]

\textbf{Path-based encoding.}
Bananas \cite{white2019bananas} proposes the path-based encoding for neural network architectures. The drawback of path-based encoding is that we need to enumerate all possible paths from the input node to the output node, in terms of the operations. This can potentially raise although it can work well in NASBench dataset \cite{ying2019bench} which results in $ \sum^5_{i=0} 3^i = 364$ possible paths.

%The total number of paths is  $ \sum^n_{i=0} q^i = 364$ where $n$ denotes the number of nodes in the cell, and $q$ denotes the number of operations for each node. For example, the NASBench \cite{ying2019bench} dataset has $ \sum^5_{i=0} 3^i = 364$ possible paths.

%Architecture featurization via path-based encoding. Prior work aiming to encode or featurize neural networks have proposed using an adjacency matrix-based approach \cite{wang2019alphax,ying2019bench,deng2017peephole,baker2017accelerating}.
% A path encoding of a cell is created by enumerating all possible paths from the input node to the output node, in terms of the operations (see Figure 1). The total number of paths is  $ \sum^n_{i=0} q^i = 364$ where $n$ denotes the number of nodes in the cell, and $q$ denotes the number of operations for each node. For example, the NASBench \cite{ying2019bench} dataset has $\sum^5_{i=0} 3^i= 364$ possible paths.

\textbf{Kernel graph.}
Previous work considers the neural network architectures as the graphs, then  defining various distances and kernels on graphs \cite{messmer1998new,wallis2001graph,kondor2002diffusion,smola2003kernels,gao2010survey,ru_iclr2021}. However, they may not be ideal for our NAS setting because neural networks have additional complex properties in addition to graphical structure, such as the type of operations performed at each layer, the number of neurons, etc. Some methods do allow different vertex sets \cite{vishwanathan2010graph}, they cannot handle layer masses and layer similarities. 

%%%%%%%%%%%%%%%%%%%
\section{Related works in batch neural architecture search}\label{sec:app:batch}

They are several approaches in the literature which can be used to select multiple architectures for evaluation, including monte carlo tree search \cite{alphax}, evolutionary search \cite{real2019regularized} and most of the batch Bayesian optimization approaches, such as Krigging believer \cite{Ginsbourger_2010Kriging}, GP-BUCB \cite{Desautels_2014Parallelizing}, B3O \cite{Nguyen_ICDM2016Budgeted}, GP-Thompson Sampling \cite{Hernandez_2017Parallel}, and BOHB \cite{falkner2018bohb}.

Krigging believer (KB) \cite{Ginsbourger_2010Kriging} exploits an interesting fact about GPs: the predictive variance of GPs depends only on the input $\bx$,
but not the outcome values $y$. KB will iteratively construct a batch of points. First, it finds the maximum of the acquisition function, like the sequential setting.
Next, KB moves to the next maximum by suppressing
this point. This is done by inserting the outcome at
this point as a halucinated value. This process is repeated until the batch is filled.

GP-BUCB \cite{Desautels_2014Parallelizing} is related to the above Krigging believer in exploiting the GP predictive variance. Particularly, GP-BUCB is similar to KB when the halucinated value is set to the GP predictive mean.

%We also compare with the vanilla k-DPP (without using quality) \cite{Kathuria_NIPS2016Batched}.

GP-Thompson Sampling \cite{Hernandez_2017Parallel} generates a batch of points by drawing from the posterior distribution of the GP to fill in a batch. In the continuous setting, we can draw a GP sample using random Fourier feature \cite{Rahimi_2007Random}. In our discrete case of NAS, we can simply draw samples from the GP predictive mean.

%%%%%%%%%%%%%%%%%
\section{Datasets}\label{sec:app:datasets}

We summarize two benchmark datasets used in the paper. Neural architecture search (NAS) methods are notoriously difficult to reproduce and compare due to different search spaces, training procedures and computing cost. These make methods inaccessible to most researchers. Therefore, the below two benchmark datasets have been created.
\paragraph{NASBENCH101.}

 The NAS-Bench-101 dataset\footnote{\url{https://github.com/google-research/nasbench}} contains over $423,000$ neural architectures with precomputed training, validation, and test accuracy \cite{ying2019bench}. In NASBench dataset, the neural network architectures have been exhaustively trained and evaluated on CIFAR-10 to create a queryable dataset. Each architecture training takes approximately $0.7$ GPU hour.

\paragraph{NASBENCH201.}
NAS-Bench-201\footnote{\url{https://github.com/D-X-Y/NAS-Bench-201}} includes all possible architectures generated by $4$ nodes and $5$ associated operation options, which results in $15,625$ neural cell candidates in total. The Nasbench201 dataset includes the tabular results for three subdatasets including CIFAR-10, CIFAR-100 and ImageNet-16-120. Each architecture training for Cifar10 takes approximately $0.7$ GPU hours, Cifar100 takes $1.4$ GPU hours and Imagenet takes $4$ GPU hours.

%%%%%%%%%%%%%%%%%%%%
\section{Proofs}\label{sec:app:proofs}

\subsection{Proof for Lemma \ref{lem:GP_psd}}\label{sec:proof_lem:GP_psd}

\begin{proof}
We consider $X \sim GP(m(), k())$. 
If $k$ is not a p.s.d. kernel, then there is some set of $n$ points $\left(t_i \right) _{i=1}^n $ and corresponding weights $\alpha_i \in \mathcal{R}$ such that
\begin{align}
\sum_{i=1}^n \sum_{j=1}^n \alpha_i k(t_i,t_j) \alpha_j < 0.
\end{align}

%Now, consider the joint distribution of $(X(t_i))$.
By the GP assumption, $\cov \bigl( X(t_i),X(t_j) \bigr) = k(t_i,t_j)$, we show that the variance is now negative
\begin{align}
\var \left( \sum_{i=1}^n \alpha_i X(t_i) \right) = \sum_{i=1}^n \sum_{j=1}^n \alpha_i \cov \bigl( X(t_i), X(t_j) \bigr) \alpha_j <0.
\end{align}
The negative variance concludes our prove that the GP is no longer valid with non-p.s.d. kernel.

\end{proof}

\subsection{Proof for Proposition \ref{pro:dNN}}\label{sec:proof_pro:dNN}

\begin{proof}
We have that tree-Wasserstein (TW) is a metric and negative definite \cite{le2019tree}. Therefore, $\dOperation, \dIn, \dOut$ are also a metric and negative definite.

Moreover, the discrepancy $d_{\dNN}$ is a convex combination with positive weights for $\dOperation, \dIn, \dOut$. Therefore, it is easy to verify that for given neural networks $\bx_1, \bx_2, \bx_3$, we have: 
\begin{itemize}
\item $d_{\dNN}(\bx_1, \bx_1) = 0$,
\item $d_{\dNN}(\bx_1, \bx_2) = d_{\dNN}(\bx_2, \bx_1)$,
\item $d_{\dNN}(\bx_1, \bx_2) + d_{\dNN}(\bx_1, \bx_2) \ge d_{\dNN}(\bx_2, \bx_3)$. 
\end{itemize}
Thus, $d_{\dNN}$ is a pseudo-metric. Additionally, a convex combination with positive weights preserves the negative definiteness. Therefore, $d_{\dNN}$ is negative definite.
\end{proof}

\subsection{Tree-Wasserstein kernel for neural networks}

\begin{proposition}\label{prop:TWK_ID}
Given the scalar length-scale parameter $\sigma^2_l$, the tree-Wasserstein kernel for neural networks $k(x, z) = \exp(-\frac{d_{\dNN}(x, z)}{\sigma^2_l})$ is infinitely divisible.
\end{proposition}

\begin{proof}
%Given two neural networks $\bx$ and $\bz$, we introduce new kernels $k_{\gamma}(\bx, \bz) = \exp(-\frac{d_{\dNN}(\bx, \bz)}{\gamma \sigma^2_l} )$ for $\gamma \in \NN^{*}$. Following \cite{berg1984harmonic} (Theorem 3.2.2, p.74), $k_{\gamma}(\cdot, \cdot)$ is also positive definite. Moreover, we also have $k(\bx, \bz) = \left(k_{\gamma}(\bx, \bz) \right)^{\gamma}$. Then, following \cite{berg1984harmonic} (Definition 2.6, p.76), we complete the proof.
Given two neural networks $\bx$ and $\bz$, we introduce new kernels $k_{\gamma}(\bx, \bz) = \exp \bigl(-\frac{d_{\dNN}(\bx, \bz)}{\gamma \sigma^2_l} \bigr)$ for $\gamma \in \NN^{*}$. Following \citet{berg1984harmonic} (Theorem 3.2.2, p.74), $k_{\gamma}(x, z)$ is also positive definite. Moreover, we also have $k(\bx, \bz) = \bigl(k_{\gamma}(\bx, \bz) \bigr)^{\gamma}$. Then, following \citet{berg1984harmonic} (Definition 2.6, p.76), we complete the proof.
\end{proof}

From Proposition~\ref{prop:TWK_ID}, one does not need to recompute the Gram matrix of the TW kernel for each choice of $\sigma^2_l$, since it suffices to compute it once.

%\section{Additional information on the tree-Wasserstein}

%General search space of directed acyclic graphs for cell-based NAS methods. Exhaustively trained and evaluated all models on CIFAR-10 to create a queryable dataset

% For a fair and reproducible experiments in the future (Li \& Talwalkar, 2019; Sciuto et al., 2019; Ying et al., 2019), the NASBench dataset was created, which contains over 400k neural architectures with precomputed training, validation, and test accuracy \cite{ying2019bench}.

%%%%%%%%%%%%%%%%%%%%%%
\section{A brief discussion on tree-Wasserstein geometry}\label{sec:app:TW}

Tree-Wasserstein is a special case of the optimal transport (OT) where the ground cost is the tree metric \cite{do2011sublinear, le2019tree}. Tree-(sliced)-Wasserstein is a generalized version of the sliced-Wassersttein (SW) which projects supports into a line (i.e., one dimensional space) and relies on the closed-form solution of the univariate OT.\footnote{When a tree is a chain, the tree-(sliced)-Wasserstein is equivalent to the sliced-Wasserstein.} Moreover, the tree-Wasserstein alleviates the curse of dimensionality for the sliced-Wasserstein.\footnote{SW projects supports into one dimensional space which limits its capacity to capture structures of distributions. The tree-Wasserstein alleviates this effect since a tree is more flexible and has a higher degree of freedom than a chain.} Moreover, the tree-Wasserstein has a closed-form for computation and is negative definite which supports to build positive definite kernels.\footnote{The standard OT is in general indefinite.}

The tree structure has been leveraged for various optimal transport problems, especially for large-scale settings such as the unbalanced OT problem where distributions have different total mass~\cite{le2021ept}, the Gromov-Wasserstein problem where supports of distributions are in different spaces~\cite{le2021fba, memoli2021ultrametric}, Gromov-Hausdoff problem for metric measure spaces~\cite{memoli2019gromov}, or the Wasserstein barycenter problem which finds the closest distribution to a given set of distributions~\cite{le2019wb}.

In our work, we propose tree-Wasserstein for neural network architectures by leveraging a novel design of tree structures on network operations, indegree and outdegree representation for network structure. Therefore, our tree-Wasserstein for neural network can capture not only the frequency of layer operations (i.e., $n$-gram representation for layer operations) but also the entire network structure (i.e., indegree and outdegree representation).

%%%%%%%%%%%%%%%%%%%%%%
\section{An example of TW computation for neural network architectures} \label{sec:app:TW_Ex}

We would like to recall the TW distance between probability measures in Equation~\eqref{equ:OT_LT}:
\begin{align}
W_{d_{\TM}}(\omega, \nu) = \sum_{e \in \Tt} w_e \bigl| \omega \bigl(\Gamma(v_e) \bigr) - \nu \bigl( \Gamma(v_e) \bigr) \bigr|.
\end{align}
We emphasize that tree Wasserstein (TW) is computed as \emph{a sum over edges} in a tree, but \emph{not the sum over supports} (e.g., in $1$-gram representation). We next explain all components in TW (Equation~\eqref{equ:OT_LT}) for computation in details.

\begin{figure*} [t]  
     \centering
\includegraphics[trim=0cm 0.cm 0cm  0.cm, clip, width=0.7\textwidth]{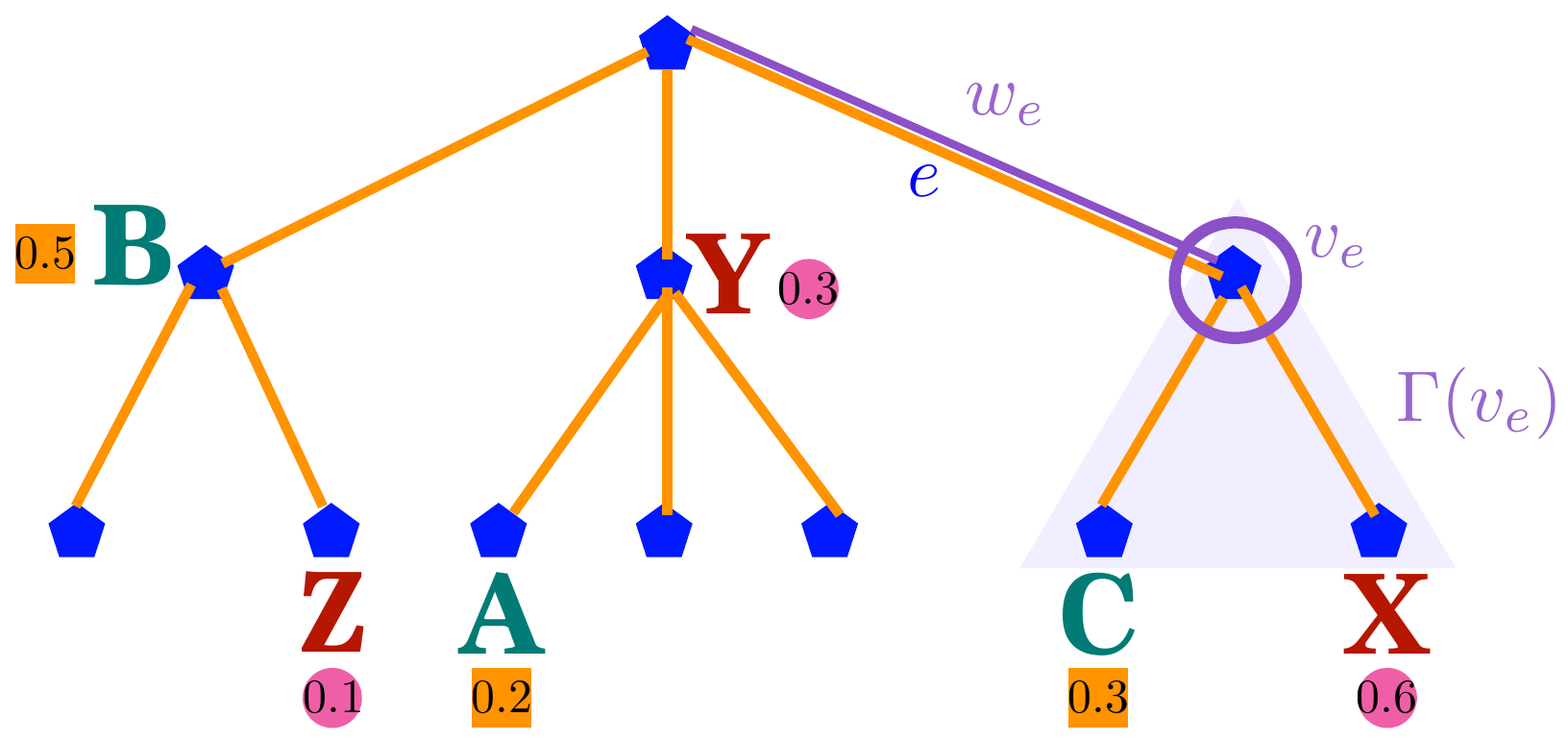}
    \caption{We illustrate a tree metric using a tree $\Tt$ which is used to explain the TW formula in Eq. \eqref{equ:OT_LT}. See text for more details.} \label{fig:example_TreeMetric}
\end{figure*}

Let consider the tree metric illustrated in Fig.~\ref{fig:example_TreeMetric}, and two probability measures $\omega = 0.2\delta_{A} + 0.5\delta_{B} + 0.3\delta_{C}$, and $\nu = 0.6\delta_{X} + 0.3\delta_{Y} + 0.1\delta_{Z}$. In this example, $A, B, C$ are supports of probability measure $\omega$ with corresponding weights $0.2, 0.5, 0.3$. In other words, the mass of the probability measure $\omega$ at each support $A, B, C$ is $0.2, 0.5, 0.3$ respectively. Similarly, $X, Y, Z$ are supports of probability measure $\nu$ with corresponding weights $0.6, 0.3, 0.1$. Next, let consider edge $e$ in Fig.~\ref{fig:example_TreeMetric} (the blue $e$), we have the following:
\begin{itemize}
    \item $w_e$ is the weight (or length) of the edge $e$ (the purple line to emphasize the edge blue $e$); 
    \item $v_e$ is one of the two nodes of the edge $e$ which is farther away from the tree root (the node with the purple circle); 
    \item $\Gamma(v_e)$ is the subtree rooted at $v_e$ (illustrated by the purple triangle blurred region); 
    \item For the probability measure $\omega$: $\omega \bigl(\Gamma(v_e) \bigr)$ is the total mass of the probability measure $\omega$ in the subtree $\Gamma(v_e)$; and only the support $C$ (with weight $0.3$) of $\omega$ is on the subtree $\Gamma(v_e)$. Hence, $\omega(\Gamma(v_e)) = 0.3$.
    \item For the probability measure $\nu$: $\nu \bigl(\Gamma(v_e) \bigr) = 0.6$. There is only the support $X$ with weight $0.6$ of $\nu$ is in the subtree $\Gamma(v_e)$.
    
\end{itemize}
% Similarly, for this considered edge $e$, we have $\nu(\Gamma(v_e)) = 0.6$ (since only the support $X$ with weight $0.6$ of $\nu$ is in the subtree $\Gamma(v_e)$. 

We do the same procedure for all other edges in the tree $\Tt$ for the TW computation. 

Intuitively, in case we consider the map $g: \omega \mapsto w_e\omega \bigl(\Gamma(e) \bigr) \mid_{e \in \Tt}$, we have $g(\omega) \in \RR^{m}$ where $m$ is the number of edges in the tree $\Tt$. The TW distance between two probability measures $\omega, \nu$ is equivalent to the $\ell_1$ distance between two $m$-dimensional (non-negative) vectors $g(\omega), g(\nu)$.

Now, we are ready to present an example of using TW for transporting two neural network architectures in Fig.~\ref{fig:example_TW101}. We consider a set $\SS$ of interest operations as follow $\SS = \left\{ \texttt{cv1, cv3, mp3} \right\}$.

\begin{figure*} [t]  
   
     \centering
  
\includegraphics[trim=0cm 0.cm 0cm  0.cm, clip, width=0.99\textwidth]{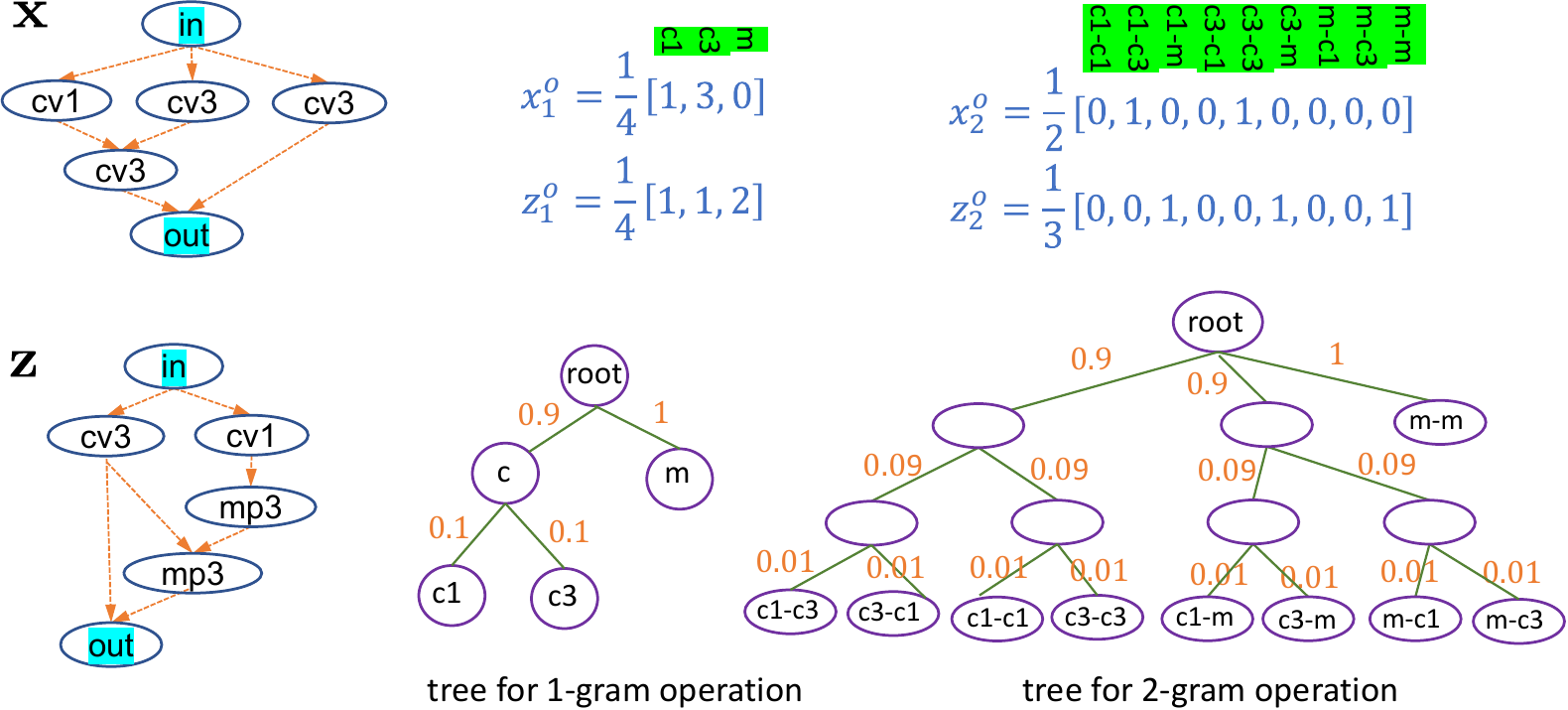}

    \caption{Example of TW used for calculating two architectures. The label of each histogram bin is highlighted in green. The distance between two nodes in a tree is sum of total cost if we travel between the two nodes, see Eq. (\ref{equ:OT_LT}). For example, the cost for moving  $\texttt{maxpool}$ ($m$) to  $\texttt{conv1}$ ($c1$) is $1+0.9+0.1=2$. We use similar analogy for computing in 2-gram (2g) representation. For the tree construction, let consider the the $1$-gram tree, we group \textit{c1} (\texttt{conv1}) and \textit{c3} (\texttt{conv3}) due to their similarity, and "\textit{c}" (\texttt{conv}) is an abstract label for the group of "convolution". We define the edge weights decreasing when the edge is far from the root, inspired by the partition-based tree metric sampling method (see the "$n$-gram representation for layer operations" part in $\S$2.2).} \label{fig:example_TW101}
\end{figure*}

% The tree weights is defined using the principle as follows: (i) identical operation (\texttt{cv1},\texttt{cv1}) has zero cost, (ii) similar operations (\texttt{cv1},\texttt{cv3}) have small cost, and (iii) very different operations (\texttt{cv1},\texttt{mp3})  have high cost.

%We consider two architectures in Fig. \ref{fig:example_TW101}.

\paragraph{Neural network information $\left( \Ss^{o}, A \right)$.} We use the order top-bottom and left-right for layers in $\Ss^{o}$.

\begin{itemize}
\item For neural network $\bx$, we have $\Ss^{o}_{\bx} = \left\{\texttt{in, cv1, cv3, cv3, cv3, out} \right\}$,

\begin{equation*}
A_{\bx} = 
\begin{pmatrix}
0 & 1 & 1 & 1 & 0 & 0 \\
0 & 0 & 0 & 0 & 1 & 0 \\
0 & 0 & 0 & 0 & 1 & 0 \\
0 & 0 & 0 & 0 & 0 & 1 \\
0 & 0 & 0 & 0 & 0 & 1 \\
0 & 0 & 0 & 0 & 0 & 0 
\end{pmatrix}
\end{equation*}

\item For neural network $\bz$, we have $\Ss^{o}_{\bz} = \left\{\texttt{in, cv3, cv1, mp3, mp3, out} \right\},$
\begin{equation*}
A_{\bz} = 
\begin{pmatrix}
0 & 1 & 1 & 0 & 0 & 0 \\
0 & 0 & 0 & 0 & 1 & 1 \\
0 & 0 & 0 & 1 & 0 & 0 \\
0 & 0 & 0 & 0 & 1 & 0 \\
0 & 0 & 0 & 0 & 0 & 1 \\
0 & 0 & 0 & 0 & 0 & 0 
\end{pmatrix}
\end{equation*}

\end{itemize}

We show how to calculate these three representations (layer operation, indegree and outdegree) using tree-Wasserstein.

\subsection{$n$-gram representation for layer operations} 

\paragraph{$\bullet$ $1$-gram representation.} The $1$-gram representations $\bx^{o}_1$ and $\bz^{o}_1$ for neural network $\bx$, and $\bz$ respectively are:
% \[
% \bx^{o}_1 = \bz^{o}_1 = \frac{1}{6} \left(1, 3, 1, 1, 0 \right), 
% \]
\begin{eqnarray*}
\bx^{o}_1 = \frac{1}{4}(1, 3, 0) \qquad
\bz^{o}_1 = \frac{1}{4}(1, 1, 2) 
\end{eqnarray*}
where we use the order (\texttt{1:cv1, 2:cv3, 3:mp3}) for the frequency of interest operations in the set $\SS$ for the $1$-gram representation of neural network.

\paragraph{$\bullet$ $2$-gram representation.} For the $2$-gram representations $\bx^{o}_2$ and $\bz^{o}_2$ for neural networks $\bx$ and $\bz$ respectively, we use the following order for $\SS \times \SS$: (\texttt{1:cv1-cv1, 2:cv1-cv3, 3:cv1-mp3, 4:cv3-cv1, 5:cv3-cv3, 6:cv3:mp3, 7:mp3-cv1, 8:mp3:cv3, 9:mp3-mp3}). Thus, we have
\begin{eqnarray}
 \bx^{o}_2 = \frac{1}{2}(0, 1, 0, 0, 1, 0, 0, 0, 0),
\qquad
 \bz^{o}_2 = \frac{1}{3}(0, 0, 1, 0, 0, 1, 0, 0, 1). \nonumber
\end{eqnarray}

Or, we can represent them as empirical measures
\begin{eqnarray}
& \omega_{\bx^{o}_2} = \frac{1}{2} \delta_{\texttt{cv1-cv3}} +  \frac{1}{2} \delta_{\texttt{cv3-cv3}}, \hspace{+47pt}
\qquad
& \omega_{\bz^{o}_2} = \frac{1}{3} \delta_{\texttt{cv1-mp3}} + \frac{1}{3} \delta_{\texttt{cv3-mp3}} + \frac{1}{3} \delta_{\texttt{mp3-mp3}}. \nonumber
\end{eqnarray}

\paragraph{$\bullet$ Tree metrics for $n$-gram representations for layer operations.} We can use the tree metric in Fig.~\ref{fig:example_TW101} for $1$-gram and $2$-gram representations. The tree metric for operations are summarized in Table~\ref{tb:TM_TW_Operation_NB101}. 

\begin{table}[]
\begin{center}
\caption{The tree-metric for operations in $\dOperation$ from our predefined tree in Fig.~\ref{fig:example_TW101} satisfying the following properties: (i) identical operation
(\texttt{cv1},\texttt{cv1}) has zero cost, (ii) similar operations (\texttt{cv1},\texttt{cv3}) have small cost, and (iii) very different operations (\texttt{cv1},\texttt{mp3})
 have high cost. Then, the similarity score
 is normalized in Eq. (\ref{equ:dNN}) and Eq. (\ref{equ:TWKernel_NN}) in which the weighting parameters $\alpha$ are learnt directly from the data. The utility score is further standardized $\mathcal{N}(0, 1)$ for robustness. Therefore, our model is robust to the choice of the predefined tree for the cost in this table. (Recall that the cost in the table is computed by tree metric (i.e., the length of the shortest path) between the pair-wise operations in the predefined tree in Fig.~\ref{fig:example_TW101}.)}
\label{tb:TM_TW_Operation_NB101}
\vspace{5pt}
\begin{tabular}{|c|c|c|c|}
\hline
    & \texttt{cv1} & \texttt{cv3} & \texttt{mp3} \\ \hline
\texttt{cv1} & $0$   & $0.2$ & $2$   \\ \hline
\texttt{cv3} & $0.2$ & $0$   & $2$   \\ \hline
\texttt{mp3} & $2$   & $2$   & $0$   \\ \hline
\end{tabular}
\end{center}
\end{table}

% These rules are intuitive and understandable for the deep learning community. 

Using the closed-form computation of tree-Wasserstein presented in Eq.~\eqref{equ:OT_LT} in the main text, we can compute  $\dOperation(\bx^{o}_1,\bz^{o}_1)$ for $1$-gram representation and $\dOperation(\bx^{o}_2,\bz^{o}_2)$ for $2$-gram representation.

For $1$-gram representation, we have %\vcom{check}
\begin{equation}
\dOperation(\bx^{o}_1,\bz^{o}_1) = 0.1\left|\frac{1}{4} - \frac{1}{4} \right| + 0.1\left|\frac{3}{4} - \frac{1}{4} \right| + 0.9\left|1 - \frac{2}{4} \right| + 1\left|0 - \frac{2}{4} \right| = 1.
\end{equation}

For $2$-gram representation, we have
\begin{align}
\dOperation(\bx^{o}_2,\bz^{o}_2) = & 
0.1\left|\frac{1}{2} - 0 \right| + 0.1\left|\frac{1}{2} - 0 \right| + 0.9\left|1 - 0 \right| \\ \nonumber
& +0.01\left|0 - \frac{1}{3}\right| + 0.01\left|0 - \frac{1}{3}\right| + 0.99\left|0 - \frac{2}{3}\right| + 1\left|0 - \frac{1}{3}\right|= 2. 
\end{align}

\subsection{Indegree and outdegree representations for network structure}

%\Tam{@Vu: please derive this example following your new formula (it is not equivalent with the one I wrote before. So, maybe it is better if you do this part. I am not sure I understand your ideas clearly yet.}

The indegree and outdegree empirical measures $(\omega^{d^-}_{\bx}, \omega^{d^+}_{\bx})$ and $(\omega^{d^-}_{\bz}, \omega^{d^+}_{\bz})$ for neural networks $\bx$ and $\bz$ respectively are:
\begin{eqnarray}
&& \omega^{d^-}_{\bx} = \sum_{i=1}^{6} \bx^{d^-}_i \delta_{\frac{\eta_{\bx, i} + 1}{M_{\bx}+ 1}}, \qquad \omega^{d^+}_{\bx} = \sum_{i=1}^{6} \bx^{d^+}_i \delta_{\frac{\eta_{\bx, i} + 1}{M_{\bx}+ 1}} \\
&& \omega^{d^-}_{\bz} = \sum_{i=1}^{6} \bz^{d^-}_i \delta_{\frac{\eta_{\bz, i} + 1}{M_{\bz}+ 1}}, \qquad \omega^{d^+}_{\bz} = \sum_{i=1}^{6} \bz^{d^+}_i \delta_{\frac{\eta_{\bz, i} + 1}{M_{\bz}+ 1}},
\end{eqnarray}
where
\begin{eqnarray}
&& \bx^{d^-} = \left(0, \frac{1}{7}, \frac{1}{7}, \frac{1}{7}, \frac{2}{7},  \frac{2}{7} \right), \qquad \bx^{d^+} = \left(\frac{3}{7}, \frac{1}{7}, \frac{1}{7}, \frac{1}{7}, \frac{1}{7},  0 \right), \\
&& \bz^{d^-} = \left(0, \frac{1}{7}, \frac{1}{7}, \frac{1}{7}, \frac{2}{7},  \frac{2}{7} \right), \qquad \bz^{d^+} = \left(\frac{2}{7}, \frac{2}{7}, \frac{1}{7}, \frac{1}{7}, \frac{1}{7},  0 \right), \\
&& \eta_{\bx} = \left(0, 1, 1, 1, 2, 3\right), \qquad \qquad \eta_{\bz} = \left(0, 1, 1, 2, 3, 4\right),
\end{eqnarray}
$M_{\bx} = 3$, $M_{\bz} = 4$, and $\bx^{d^-}_i, \bx^{d^+}_i, \eta_{\bx, i}$ are the $i^{th}$ elements of $\bx^{d^-}, \bx^{d^+}, \eta_{\bx}$ respectively. Consequently, one can leverage the indegree and outdegree for network structures to distinguish between $\bx$ and $\bz$.

%\Tam{@Vu: please double-check and rewrite according to your new formula for the indegree and outdegree part}

 We demonstrate in Fig. \ref{fig:TW_InOut} how to calculate the tree-Wasserstein for indegree and outdegree.
The supports of empirical measures $\omega^{d^{-}}_{\bx}$ and $\omega^{d^{-}}_{\bz}$ are in a line. For simplicity, we choose a tree as a chain of real values for the tree-Wasserstein distance. The tree becomes a chain of increasing real values, i.e., $\frac{1}{4} \rightarrow \frac{2}{4} \rightarrow \frac{3}{4} \rightarrow 1$, the weight in each edge is the $\ell_1$ distance between two nodes of that edge. Particularly, the tree-Wasserstein is equivalent to the univariate optimal transport. It is similar for empirical measures $\omega^{d^{+}}_{\bx}$ and $\omega^{d^{+}}_{\bz}$.

\begin{figure*} [t]  
     \centering
    % trim={<left> <lower> <right> <upper>}
    \includegraphics[trim=0cm 0.cm 0cm  0.cm, clip, width=0.8\linewidth]{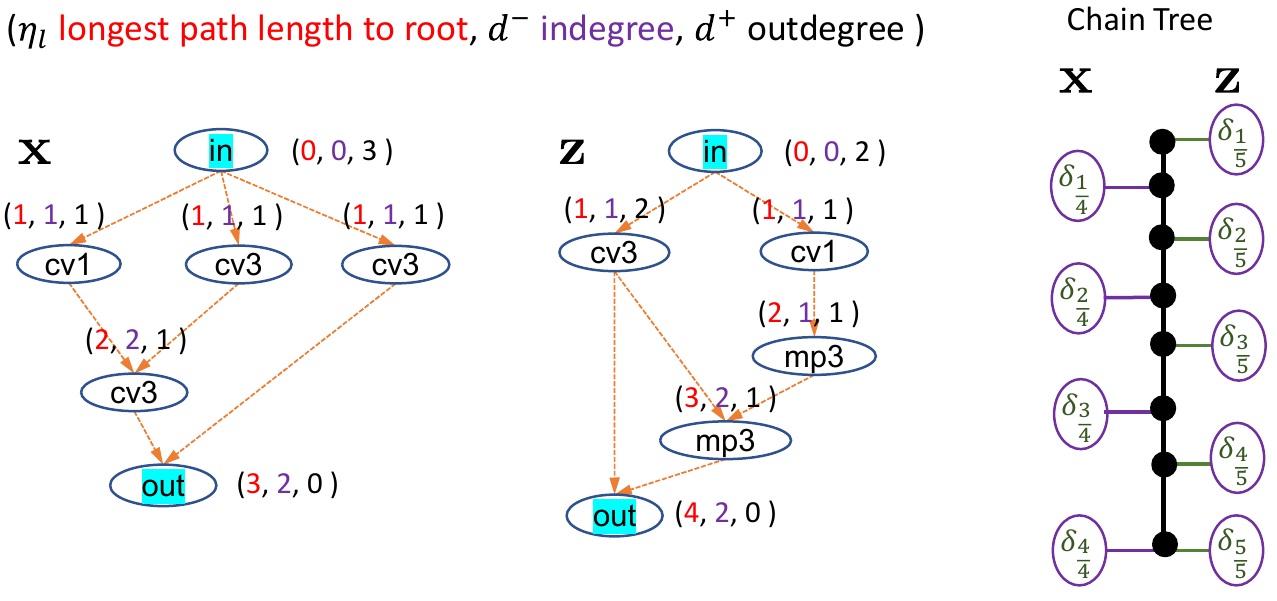}

    \caption{Illustration of indegree and outdegree used in TW. Let $\eta_{x,\ell}$ and $M_x$ be lengths of the longest paths from an input layer to a layer $\ell$ and to an output layer respectively. We can represent the empirical measure as $\omega^{d^-}_{\bx} = \sum_{\ell \in L_{\bx}} \bx^{d^{-}}_{\ell} \delta_{\frac{\eta_{x,\ell}+1}{M_x+1}}= \frac{3}{7}\delta_{\frac{2}{4}}+\frac{2}{7}\delta_{\frac{3}{4}} + \frac{2}{7}\delta_{\frac{4}{4}}$ and $\omega^{d^+}_{\bx} = \sum_{\ell \in L_{\bx}} \bx^{d^{+}}_{\ell} \delta_{\frac{\eta_{x,\ell}+1}{M_x+1}}= \frac{3}{7}\delta_{\frac{1}{4}}+\frac{3}{7}\delta_{\frac{2}{4}}+\frac{1}{7}\delta_{\frac{3}{4}}$ and for $\bz$ as $\omega^{d^-}_{\bz} = \sum_{\ell \in L_{\bz}} \bz^{d^{-}}_{\ell} \delta_{\frac{\eta_{z,\ell}+1}{M_z+1}}= \frac{2}{7}\delta_{\frac{2}{5}}+\frac{1}{7}\delta_{\frac{3}{5}}+\frac{2}{7}\delta_{\frac{4}{5}}+\frac{2}{7}\delta_{\frac{5}{5}}$ and $\omega^{d^+}_{\bx} = \sum_{\ell \in L_{\bz}} \bz^{d^{+}}_{\ell} \delta_{\frac{\eta_{z,\ell}+1}{M_z+1}}= \frac{2}{7}\delta_{\frac{1}{5}}+\frac{3}{7}\delta_{\frac{2}{5}}+\frac{1}{7}\delta_{\frac{3}{5}}+\frac{1}{7}\delta_{\frac{4}{5}}$. The tree, which is a chain (1/5 $\rightarrow$ 1/4 $\rightarrow$ 2/5 $\rightarrow$ 2/4 $\rightarrow$ 3/5 $\rightarrow$ 3/4 $\rightarrow$ 4/5 $\rightarrow$ 1), is used to compute the distance. } \label{fig:TW_InOut}
\end{figure*}

\paragraph{$\bullet$ $\dIn(\omega_{\bx}^{d^{-}},\omega_{\bz}^{d^{-}})$ for indegree representation.}

Using Eq. (\ref{equ:OT_LT}) we have
% \begin{equation}
% \dIn(\omega_{\bx}^{d^{-}},\omega_{\bz}^{d^{-}}) = \left|\frac{1}{3} - \frac{2}{3} \right| \frac{1}{7} = \frac{1}{21} = 0.0476. 
% \end{equation}
% \begin{equation}
% \dIn(\omega_{\bx}^{d^{-}},\omega_{\bz}^{d^{-}}) = \frac{1}{3}\underbrace{  \left|\frac{4}{7} - \frac{5}{7} \right|}_{ \omega_{\cdot}^{d-}\left(\Gamma\left(\delta_\frac{3}{3}\right)\right)} +   \frac{1}{3} \underbrace{\left|\frac{3}{7} - \frac{2}{7} +  \underbrace{ \frac{4}{7} - \frac{5}{7}}_{ \omega_{\cdot}^{d-}\left(\Gamma\left(\delta_\frac{3}{3}\right)\right)}\right|}_{\omega_{\cdot}^{d-}\left(\Gamma\left(\delta_\frac{2}{3}\right)\right)}  = \frac{1}{21} = 0.0476. 
% \end{equation}
% \begin{equation}
% \dIn(\omega_{\bx}^{d^{-}},\omega_{\bz}^{d^{-}}) = \frac{1}{3}\underbrace{  \left|\frac{4}{7} - \frac{5}{7} \right|}_{ \omega_{\cdot}^{d-}\left(\Gamma\left(\delta_\frac{3}{3}\right)\right)} +   \frac{1}{3} \underbrace{\left| \left(\frac{3}{7} +  \frac{4}{7}\right) - \left(\frac{2}{7} + \frac{5}{7}\right) \right|}_{\omega_{\cdot}^{d-}\left(\Gamma\left(\delta_\frac{2}{3}\right)\right)}  = \frac{1}{21} = 0.0476. 
% \end{equation}
% \frac{1}{3}\underbrace{  \left|\frac{4}{7} - \frac{5}{7} \right|}_{ \Gamma\left(\delta_\frac{3}{3}\right)}
\begin{eqnarray}
\dIn(\omega_{\bx}^{d^{-}},\omega_{\bz}^{d^{-}}) & = \left(\frac{1}{4} - \frac{1}{5}\right)\underbrace{  \left|\frac{7}{7} - \frac{7}{7} \right|}_{ \Gamma\left(\delta_\frac{1}{4}\right)} + \left(\frac{2}{5} - \frac{1}{4}\right)\underbrace{  \left|\frac{7}{7} - \frac{7}{7} \right|}_{ \Gamma\left(\delta_\frac{2}{5}\right)} + \left(\frac{2}{4} - \frac{2}{5}\right)\underbrace{  \left|\frac{7}{7} - \frac{5}{7} \right|}_{ \Gamma\left(\delta_\frac{2}{4}\right)} + \left(\frac{3}{5} - \frac{2}{4}\right)\underbrace{  \left|\frac{4}{7} - \frac{5}{7} \right|}_{ \Gamma\left(\delta_\frac{3}{5}\right)} \nonumber \\
& + \left(\frac{3}{5} - \frac{2}{4}\right)\underbrace{  \left|\frac{4}{7} - \frac{4}{7} \right|}_{ \Gamma\left(\delta_\frac{3}{4}\right)} + \left(\frac{3}{5} - \frac{2}{4}\right)\underbrace{  \left|\frac{2}{7} - \frac{4}{7} \right|}_{ \Gamma\left(\delta_\frac{4}{5}\right)} + \left(\frac{3}{5} - \frac{2}{4}\right)\underbrace{  \left|\frac{2}{7} - \frac{2}{7} \right|}_{ \Gamma\left(\delta_1\right)} = \frac{4}{70} = 0.0571
\end{eqnarray}

where for each value $\left( t_1 - t_2 \right)\left| a_1 - a_2 \right|$, the value in the parenthesis, e.g., $\left( t_1 - t_2 \right)$, is defined as the edge weights from $\delta_{t_1}$ to $\delta_{t_2}$ in Fig. \ref{fig:TW_InOut} and the values in the absolute difference, e.g., $\left| a_1 - a_2 \right|$ is the total mass of empirical measures in the subtrees rooted as the deeper node (i.e., $\delta_{t_2}$) of corresponding edge (from $\delta_{t_1}$ to $\delta_{t_2}$) as defined in Eq. (\ref{equ:OT_LT}).

%It is special in the second bracket that the cost of $\delta_{\frac{2}{3}}$ will involve an additional cost to the child node $\delta_{\frac{3}{3}}$, as the property of tree-Wasserstein.

\paragraph{$\bullet$ $\dOut(\omega_{\bx}^{d^{+}},\omega_{\bz}^{d^+})$ for outdegree representation.}
Similarly, for outdegree representation, we have
% \begin{equation}
% \dOut(\omega_{\bx}^{d^{+}},\omega_{\bz}^{d^{+}}) = \left|\frac{1}{3} - 1 \right| \frac{1}{7} = \frac{2}{21} = 0.0952. 
% \end{equation}

% \begin{align}
% \dOut(\omega_{\bx}^{d^{+}},\omega_{\bz}^{d^{+}}) &= \frac{1}{3} \underbrace{\left|\frac{1}{7} - \frac{2}{7} \right|}_{\omega_{\cdot}^{d+}\left(\Gamma\left(\delta_\frac{3}{3}\right)\right)} +  \frac{1}{3} \underbrace{\left| \left(\frac{3}{7} +  \frac{1}{7}\right) + \left(\frac{3}{7} - \frac{2}{7} \right)\right|}_{\omega_{\cdot}^{d+}\left(\Gamma\left(\delta_\frac{2}{3}\right)\right)} = \frac{2}{21} = 0.0952. 
% \end{align}
\begin{eqnarray}
\dOut(\omega_{\bx}^{d^{+}},\omega_{\bz}^{d^{+}}) & = \left(\frac{1}{4} - \frac{1}{5}\right)\underbrace{  \left|\frac{7}{7} - \frac{5}{7} \right|}_{ \Gamma\left(\delta_\frac{1}{4}\right)} + \left(\frac{2}{5} - \frac{1}{4}\right)\underbrace{  \left|\frac{4}{7} - \frac{5}{7} \right|}_{ \Gamma\left(\delta_\frac{2}{5}\right)} + \left(\frac{2}{4} - \frac{2}{5}\right)\underbrace{  \left|\frac{4}{7} - \frac{2}{7} \right|}_{ \Gamma\left(\delta_\frac{2}{4}\right)} + \left(\frac{3}{5} - \frac{2}{4}\right)\underbrace{  \left|\frac{1}{7} - \frac{2}{7} \right|}_{ \Gamma\left(\delta_\frac{3}{5}\right)} \nonumber \\
& + \left(\frac{3}{5} - \frac{2}{4}\right)\underbrace{  \left|\frac{1}{7} - \frac{1}{7} \right|}_{ \Gamma\left(\delta_\frac{3}{4}\right)} + \left(\frac{3}{5} - \frac{2}{4}\right)\underbrace{  \left|\frac{0}{7} - \frac{1}{7} \right|}_{ \Gamma\left(\delta_\frac{4}{5}\right)} + \left(\frac{3}{5} - \frac{2}{4}\right)\underbrace{  \left|\frac{0}{7} - \frac{0}{7} \right|}_{ \Gamma\left(\delta_1\right)} = \frac{6}{70} = 0.0857.
\end{eqnarray}

%\\ + \frac{1}{3} \underbrace{ \left| \left(\frac{3}{7} + \frac{3}{7} +   \frac{1}{7}\right) - \left(\frac{2}{7} +  \frac{3}{7}  + \frac{2}{7} \right) \right|}_{\omega_{\cdot}^{d+}\left(\Gamma\left(\delta_\frac{1}{3}\right)\right)}

%The cost of $\delta_{\frac{2}{3}}$ will involve an additional cost to the child node $\delta_{\frac{3}{3}}$ and the cost of $\delta_{\frac{1}{3}}$ will involve all child costs of $\delta_{\frac{2}{3}}$ and $\delta_{\frac{3}{3}}$.

From $\dOperation$, $\dIn$ and $\dOut$, we can obtain the discrepancy $d_{\dNN}$ between neural networks $\bx$ and $\bz$ as in Eq.~\eqref{equ:dNN} with predefined values $\alpha_1, \alpha_2, \alpha_3$.

%\vcom{remark on computing distance across different architecture size}

%\begin{remark}
%In case one represents indegree and outdegree as vectors for neural networks, popular distances such as $\ell_1$ distance can only be applied when input neural networks have the same number of layers. In our work, we represent indegree and outdegree as probability measures, and leverage tree-Wasserstein distance to compare them. Our approach can be applied for neural networks regardless its number of layers.
%\end{remark}

%In this special example, there is a unique tree for both indegree and outdegree. But there could be different trees for other examples.

% \begin{align}
% \omega^{d^-}_{\bx} = \sum_{\ell \in L_{\bx}} \bx^{d^{-}}_{\ell} \delta_{\frac{\eta_{x,\ell}+1}{M_x+1}}, \qquad \omega^{d^+}_{\bx} = \sum_{\ell \in L_{\bx}} \bx^{d^{+}}_{\ell} \delta_{\frac{\eta_{x,\ell}+1}{M_x+1}},
% \end{align}

% Python implementation for computing the chain for $\dOut(\omega_{\bx}^{d^{+}},\omega_{\bz}^{d^{+}})$ and $\dIn(\omega_{\bx}^{d^{-}},\omega_{\bz}^{d^{-}})$ is illustrated in Fig. \ref{fig:TW_Chain_Python} .

% \begin{figure} [t]  
%      \centering
%     % trim={<left> <lower> <right> <upper>}
%     \includegraphics[trim=0cm 0.cm 0cm  0.cm, clip, width=0.6\linewidth]{NeurIPS_2020/fig/TW_Chain.PNG}

%     \caption{Python code for computing TW Chain -- an important step in computing $\dIn$ and $\dOut$.} \label{fig:TW_Chain_Python}
% \end{figure}

%%%%%%%%%%%%%%%%%%%%%%%
\section{Optimizing hyperparameters in TW and GP}\label{app:sec:opt_hyper}

As equivalence, we consider $\lambda_1=\frac{\alpha_1}{\sigma^2_l}$, $\lambda_2=\frac{\alpha_2}{\sigma^2_l}$ and $\lambda_3=\frac{1-\alpha_1-\alpha_2}{\sigma^2_l}$ in Eq. (\ref{equ:TWKernel_NN}) and present the derivative for estimating the variable $\lambda$ in our kernel.
\begin{align}
\label{eq:app-CoCaBO-kernel}
    k(\mathbf{u}, \mathbf{v}) = \exp \left(-\lambda_1 \dOperation(\mathbf{u}, \mathbf{v})
    -\lambda_2 \dIn(\mathbf{u}, \mathbf{v}) -\lambda_3 \dOut(\mathbf{u}, \mathbf{v}) \right)   .
\end{align}
The hyperparameters of the kernel are optimised by maximising the log marginal likelihood (LML) of the GP surrogate
%\vcom{we should present this LML from Rasmussen book. Next, we show that taking derivaties of LML will be depending on the derivative of k. Then, we have our formula for derivatives of k.}
\begin{equation}
\label{eq:app-LML-1}
    \theta^* = \arg \max_\theta \mathcal{L}(\theta, \mathcal{D}),
\end{equation}
where we collected the hyperparameters into $\theta = \{\lambda_1, \lambda_2, \lambda_3, \sigma^2_n \}$. 
The LML \cite{Rasmussen_2006gaussian} and its derivative are defined as 
\begin{equation}
    \label{eq:app:LML}
    \mathcal{L}(\theta) = -\frac{1}{2}\mathbf{y}^\intercal \mathbf{K}^{-1}\mathbf{y} 
    - \frac{1}{2}\log |\mathbf{K}| + \text{constant}
\end{equation}
\begin{equation}
    \label{eq:app:LMLderiv}
    \frac{\partial\mathcal{L}}{\partial \theta} = \frac{1}{2}\left(\mathbf{y}^\intercal \mathbf{K}^{-1}\frac{\partial\mathbf{K}}{\partial\theta}\mathbf{K}^{-1}\mathbf{y} 
    -\text{tr}\left( \mathbf{K}^{-1} \frac{\partial\mathbf{K}}{\partial\theta} \right)\right),
\end{equation}
where $\mathbf{y}$ are the function values at sample locations and $\mathbf{K}$ is the covariance matrix of $k(\mathbf{x}, \mathbf{x}')$ evaluated on the training data.
Optimization of the LML was performed via multi-started gradient descent. 
The gradient in Eq. (\ref{eq:app:LMLderiv}) relies on the gradient of the kernel $k$ w.r.t. each of its parameters:
\begin{align}
    \frac{\partial k(\mathbf{u}, \mathbf{v}) }{\partial\lambda_1} &=
    -\dOperation(\mathbf{u}, \mathbf{v}) \times k(\mathbf{u}, \mathbf{v}) \\ 
   \frac{\partial k(\mathbf{u}, \mathbf{v}) }{\partial\lambda_2} &=
    -\dIn(\mathbf{u}, \mathbf{v}) \times k(\mathbf{u}, \mathbf{v}) \\ 
    \frac{\partial k(\mathbf{u}, \mathbf{v}) }{\partial\lambda_3} &=
    -\dOut(\mathbf{u}, \mathbf{v}) \times k(\mathbf{u}, \mathbf{v}) .
\end{align}

%As opposed to NASBOT with 11 hyper-parameters of its own, our three hyperparameters are less vulnerable to overfitting when optimizing with the GP marginal likelihood.

\subsection{Proof for Proposition \ref{pro:kDPP}}\label{sec:proof_pro:kDPP}

\begin{proof}
Let $A$ and $B$ be the training and test set respectively, we utilize the Schur complement to have
$K_{A \cup B} = K_A \times \left[ K_B - K_{BA} K_A^{-1} K_{AB} \right]$  and the probability of selecting $B$ is
\begin{align}
P \left( B \subset \mathcal{P} \mid A \right) = \frac{ \det \left( K_{A \cup B} \right) } { \det(K_A) } = \det \left(  K_B - K_{BA} K_A^{-1} K_{AB} \right) = \det \left( \sigma(B \mid A) \right).   \label{eq:cond_kdpp_gpvar} 
\end{align}
This shows that the conditioning of k-DPP is equivalent to the GP predictive variance $\sigma(B \mid A)$ in Eq. (\ref{eq:GPvar}).

\end{proof}

%%%%%%%%%%%%%%%%%%%
\section{Distance Properties Comparison} \label{sec:app:distance_property}

We summarize the key benefits of using tree-Wasserstein (n-gram) as the main distance with GP for sequential NAS and k-DPP for batch NAS in Table \ref{tab:property_comparison}. Tree-Wasserstein offers close-form computation and positive semi-definite covariance matrix which is critical for GP and k-DPP modeling.

%\vcom{Tam Le please comment more if we can write more here @Answer: I will think about this.}

\paragraph{Comparison with graph kernel.}
Besides the adjacency matrix representation, each architecture includes layer masses and operation type. We note that two different architectures may share the same adjacency matrix while they are different in operation type and layer mass. 

\paragraph{Comparison with path-based encoding.}
TW can scale well to more nodes, layers while the path-based encoding is limited to.

\paragraph{Comparison with OT approaches in computational complexity.}
In general, OT is formulated as a linear programming problem and its computational complexity is super cubic in the size of probability measures \cite{burkard1999} (e.g., using network simplex). On the other hand, TW has a closed-form computation in Eq.~\eqref{equ:OT_LT}, and its computational complexity is linear to the number of edges in the tree. Therefore, TW is much faster than OT in applications \cite{le2019tree}, and especially useful for large-scale settings where the computation of OT becomes prohibited.

%Comparing to edit distance is not suitable for different sizes of architectures.

\begin{table*}
\caption{Properties comparison across different distances for using with GP-BO and k-DPP. GW is Gromov-Wasserstein. TW is tree-Wasserstein. OT is optimal transport.}
 \centering
\begin{tabular}{lllllllll}

\toprule
 Representation &  Matrix/Graph & Path-Encoding & OT (or W) & GW &  TW \\
 % Distance &  XX & Edit \cite{navarro2001guided} & L1 & NASBOT & TW \\

\midrule
Closed-form estimation &  \cmark  & \cmark  & \xmark  & \xmark & \cmark  \\
Positive semi definite &  \cmark  & \cmark  & \xmark  & \xmark & \cmark  \\
Different architecture sizes  & \cmark , \xmark  & \xmark  & \cmark  & \cmark  & \cmark  \\
Scaling with architecture size   & \cmark   & \xmark  & \cmark   & \cmark  & \cmark \\
% Operation type and layer mass   & \xmark   & \xmark  & \cmark & \cmark  & \cmark   \\
\midrule
\bottomrule
\end{tabular}
 \label{tab:property_comparison}
\end{table*}

\begin{figure*} [t]  
   
     \centering
    % trim={<left> <lower> <right> <upper>}
    \includegraphics[trim=0cm 0.cm 0cm  0.cm, clip, width=0.32\textwidth]{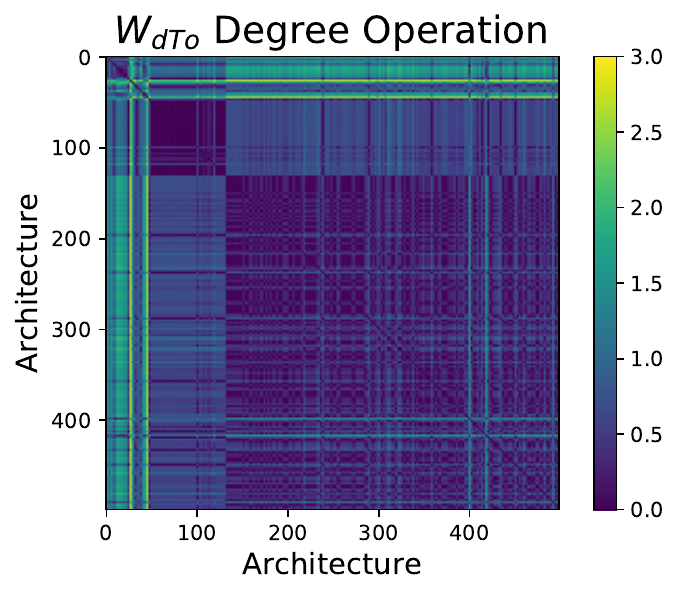}   
\includegraphics[trim=0cm 0.cm 0cm  0.cm, clip, width=0.32\textwidth]{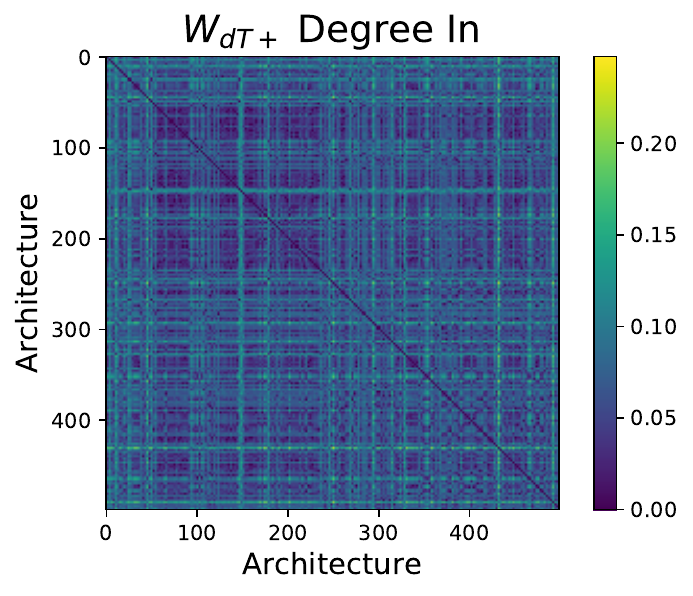}   
\includegraphics[trim=0cm 0.cm 0cm  0.cm, clip, width=0.32\textwidth]{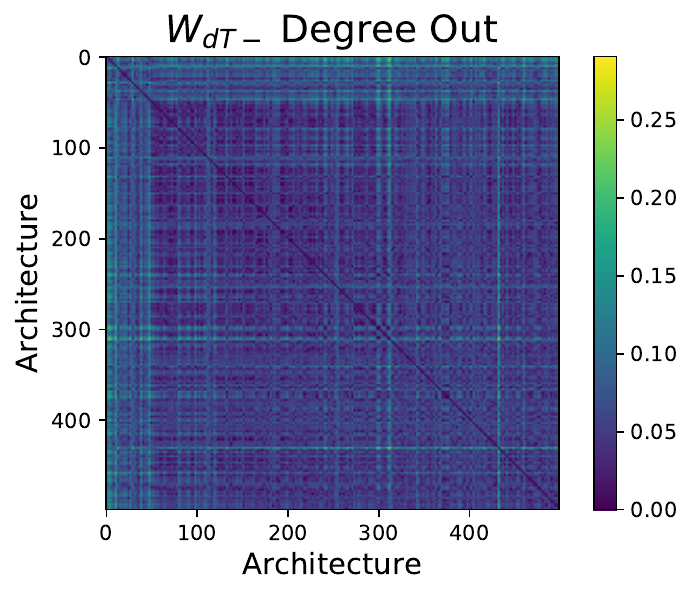}

    \caption{Tree-Wasserstein distances over 500 architectures on NASBENCH101.} \label{fig:TW_distances}
\end{figure*}

\begin{figure*} [t]  
     \centering
    % trim={<left> <lower> <right> <upper>}
\includegraphics[trim=0cm 0.cm 0cm  0.cm, clip, width=0.32\linewidth]{fig/seq_test_imagenet.pdf}
\includegraphics[trim=0cm 0.cm 0cm  0.cm, clip, width=0.32\linewidth]{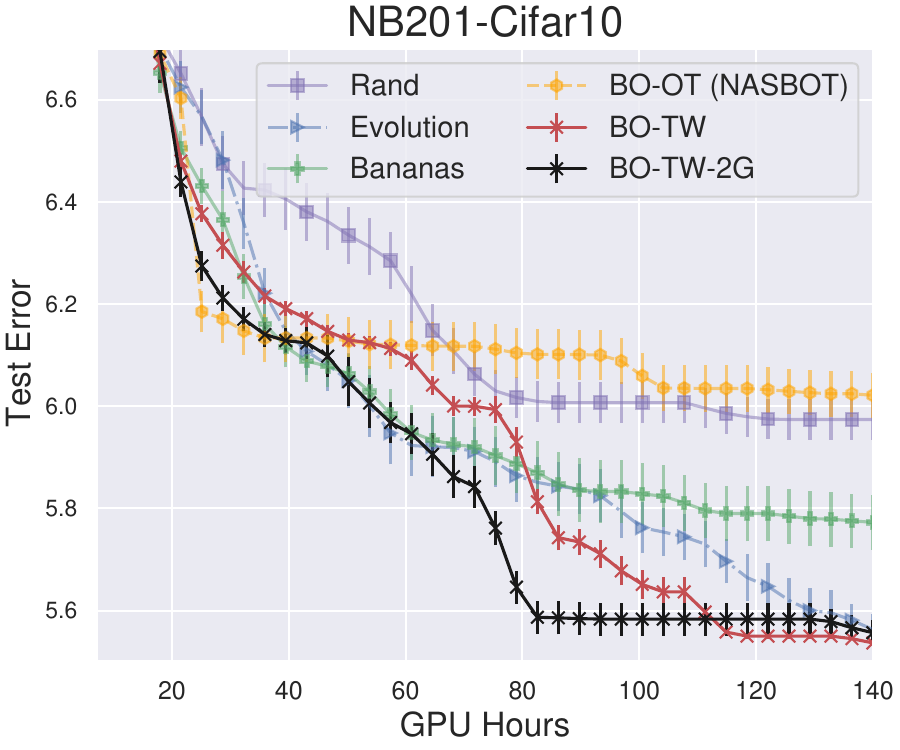}
\includegraphics[trim=0cm 0.cm 0cm  0.cm, clip, width=0.32\linewidth]{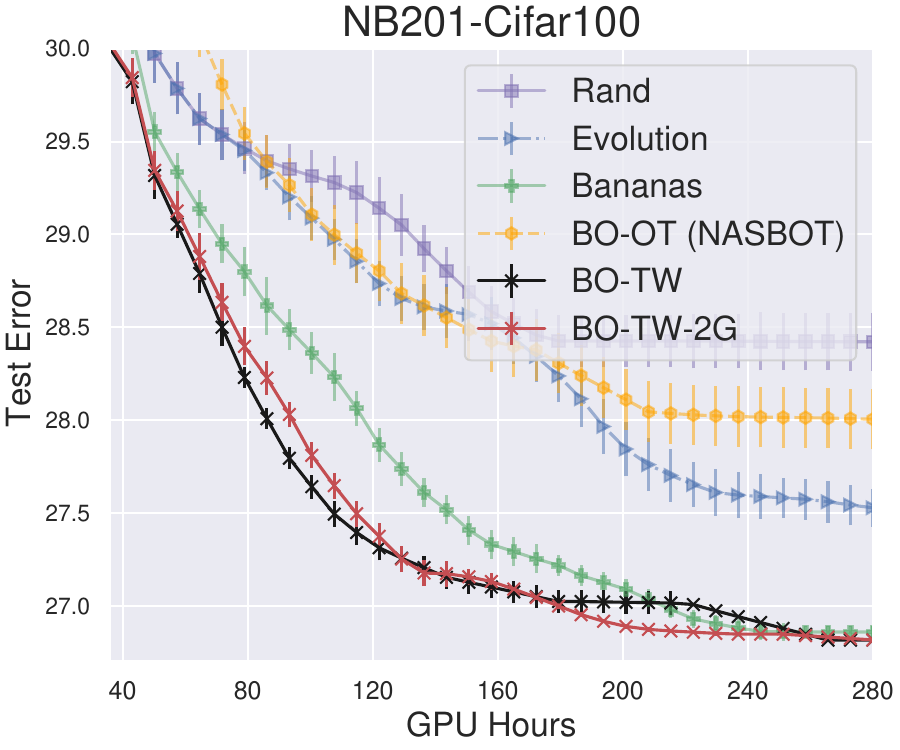}
    \caption{Additional sequential NAS comparison on NASBENCH201.} \label{fig:NASBENCH201}
\end{figure*}

%%%%%%%%%%%%%%%%%%%%%%%%
\section{Additional Experiments and Illustrations}\label{sec:app:add_experiments}

\subsection{Model Analysis} \label{sec:model_analysis}

We illustrate three internal distances of our tree-Wasserstein including $\dOperation$, $\dOut$, $\dIn$ in Fig. \ref{fig:TW_distances}. Each internal distance captures different aspects of the networks. The zero diagonal matrix indicates the correct estimation of the same neural architectures. %Given these distances, we also plot the covariance matrix $k(X,X)$  in Fig. \ref{fig:cov_mat}

% We illustrate that TW computes reasonable distances on neural network architectures via a
% two-dimensional t-SNE visualisation \cite{tsne} of the network architectures based. Given a distance matrix
% between $50$ random architectures, t-SNE embeds them in a $2$-dimensional space so that objects with small distances
% are placed closer to those that have larger distances. Fig. \ref{fig:tsne} shows the t-SNE embedding using
% the TW distance. We have indexed $3$ similar networks {A,B,C} and $3$ different networks {D,E,F} in the Bottom Fig. \ref{fig:tsne}. Similar networks A,B,C are placed close to each other
% indicating that TW induces a meaningful topology among neural network architectures.

%\begin{figure*} [t]  
%     \centering
% \includegraphics[trim=0cm 0.cm 0cm  0.cm, clip, %width=0.49\linewidth]{fig/covariance_mat.pdf}
%    \caption{Covariance matrix over 500 architectures on %NASBENCH101.} \label{fig:cov_mat}
%\end{figure*}

%\begin{wrapfigure} {r}{0.45\textwidth}
\begin{figure} %{r}{0.45\textwidth}
    \vspace{-1pt}
     \centering
    % trim={<left> <lower> <right> <upper>}
    \includegraphics[trim=0cm 0.cm 0cm  0.cm, clip, width=0.4\linewidth]{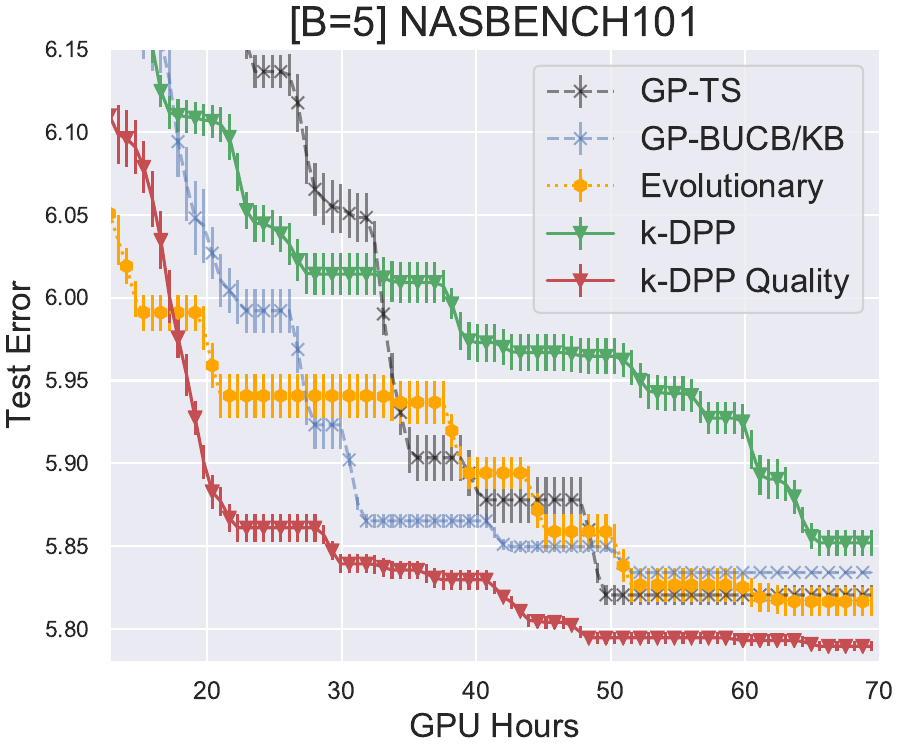}

    \caption{Additional result of batch NAS on NB101. We use TW-2G and a batch size $B=5$} \label{fig:batch_NAS_val_201}
    \vspace{-1pt}
\end{figure}

% \begin{figure*} [t]  
   
%      \centering
%     % trim={<left> <lower> <right> <upper>}
%     \includegraphics[trim=0cm 0.cm 0cm  0.cm, clip, width=0.97\textwidth]{fig/tSNE_matrix.pdf}   
%      \centering
%     % trim={<left> <lower> <right> <upper>}
    
%  \includegraphics[trim=0cm 0.cm 0cm  0.cm, clip, width=0.99\textwidth]{NeurIPS_2020/fig/tsne_ABC.pdf}  
%     \includegraphics[trim=0cm 0.cm 0cm  0.cm, clip, width=0.32\textwidth]{fig/Graph_1.pdf}   
% \includegraphics[trim=0cm 0.cm 0cm  0.cm, clip, width=0.32\textwidth]{fig/Graph_5.pdf}   
% \includegraphics[trim=0cm 0.cm 0cm  0.cm, clip, width=0.32\textwidth]{fig/Graph_7.pdf}
%  %\includegraphics[trim=0cm 0.cm 0cm  0.cm, clip, width=0.49\linewidth]{fig/covariance_mat.pdf}
 
%   \vspace{15pt}
  
%  \includegraphics[trim=0cm 0.cm 0cm  0.cm, clip, width=0.99\textwidth]{NeurIPS_2020/fig/tsne_DEF.pdf}   
% %     \includegraphics[trim=0cm 0.cm 0cm  0.cm, clip, width=0.32\textwidth]{fig/Graph_34.pdf}   
% % \includegraphics[trim=0cm 0.cm 0cm  0.cm, clip, width=0.32\textwidth]{fig/Graph_39.pdf}   
% % \includegraphics[trim=0cm 0.cm 0cm  0.cm, clip, width=0.32\textwidth]{fig/Graph_25.pdf}

%     \caption{Top: tSNE visualization of $50$ random architectures from NASBENCH101 dataset using Tree-Wasserstein. We highlight three \textbf{similar} architectures in A, B, C and three \textbf{different} architectures in D, E, F. In the bottom, we plot the architectures of A, B, C, D, E, F in details.} \label{fig:tsne}
% \end{figure*}

% We plot the best found architectures by our algorithm in Fig. \ref{fig:BestArchitectures}.

\subsection{Further sequential and batch NAS experiments}
To complement the results presented in the main paper, we present additional experiments on both sequential and batch NAS setting using NB101 and NB201 dataset in Fig. \ref{fig:NASBENCH201}. 
In addition, we present experiments on batch NAS settings in Fig. \ref{fig:batch_NAS_val_201} that the proposed k-DPP quality achieves  the best performance consistently.

% \begin{figure*} [t]  
   
%      \centering
%     % trim={<left> <lower> <right> <upper>}
%     \includegraphics[trim=0cm 0.cm 0cm  0.cm, clip, width=0.45\textwidth]{fig/Graph_2_val.pdf}   
%      \centering
%     % trim={<left> <lower> <right> <upper>}
%     \includegraphics[trim=0cm 0.cm 0cm  0.cm, clip, width=0.45\textwidth]{fig/Graph_3_val.pdf}   
    
% \includegraphics[trim=0cm 0.cm 0cm  0.cm, clip, width=0.45\textwidth]{fig/Graph_8_val.pdf}   
% \includegraphics[trim=0cm 0.cm 0cm  0.cm, clip, width=0.45\textwidth]{fig/Graph_21_val.pdf}
%  %\includegraphics[trim=0cm 0.cm 0cm  0.cm, clip, width=0.49\linewidth]{fig/covariance_mat.pdf}

%     \caption{Visualization of the best found architectures on NASBENCH101.} \label{fig:BestArchitectures}
% \end{figure*}

\subsection{Ablation study using different acquisition functions}
We evaluate our proposed model using two comon acquisition functions including UCB and EI. The result suggests that UCB tends to perform much better than EI for our NAS setting. This result is consistent with the comparison presented in Bananas \cite{white2019bananas}.

\begin{figure} %{r}{0.5\textwidth}  
     \centering
\includegraphics[trim=0cm 0.cm 0cm  0.cm, clip, width=0.48\textwidth]{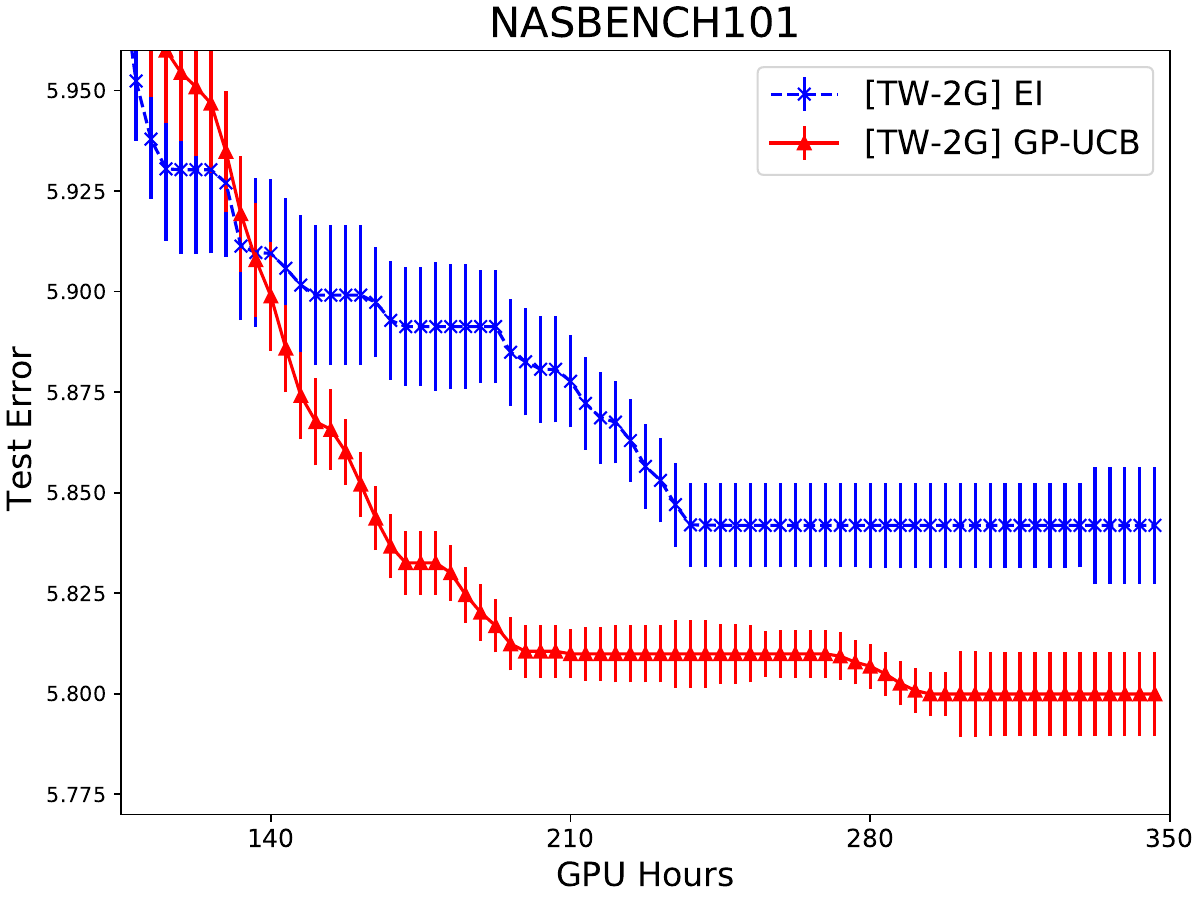}

    \caption{Optimizing the acquisition function using GP-UCB and EI on NB101. The results suggest that using GP-UCB will lead to better performance overall.} \label{fig:BestArchitectures}
    \vspace{0pt}
\end{figure}
%We provide further illustrations of the proposed TW distance.

\subsection{Ablation study with different batch size $B$}

Finally, we study the performance with different choices of batch size $B$ in Fig. \ref{fig:performance_batchB} which naturally confirms that the performance increases with larger batch size $B$.

\begin{figure}%{r}{0.33\textwidth}
    %\vspace{-20pt}
  \begin{center}
    \includegraphics[width=0.5\textwidth]{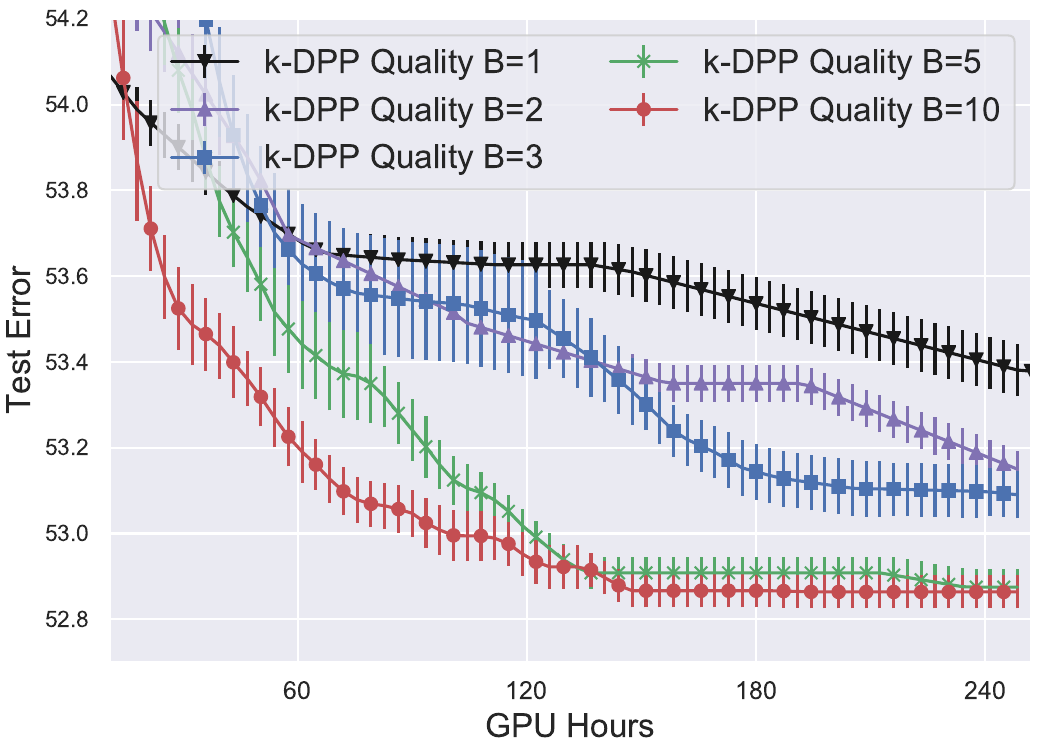}
  \end{center}
    \caption{Performance with different batch sizes $B$ on Imagenet. The result shows that the performance increases with larger batch size $B$, given the same wall-clock time budget (or batch iteration). } \label{fig:performance_batchB}
   % \vspace{-15pt}
\end{figure}
\end{document}